\documentclass{article}

\usepackage[main, final]{neurips_2026}
\usepackage[utf8]{inputenc}
\usepackage[T1]{fontenc}
\usepackage{amsmath,amssymb,amsthm,mathtools,amsfonts}
\usepackage{booktabs}
\usepackage{microtype}
\usepackage{xcolor}
\usepackage{enumitem}
\usepackage{array}
\usepackage{nicefrac}
\usepackage{url}
\usepackage{hyperref}
\usepackage{algorithm}
\usepackage{algpseudocode}
\usepackage{booktabs}
\usepackage{multirow}
\hypersetup{colorlinks=true,linkcolor=blue!55!black,citecolor=blue!55!black,urlcolor=blue!55!black}
\setcitestyle{numbers,square,comma,sort&compress}
\newtheorem{theorem}{Theorem}[section]
\newtheorem{proposition}[theorem]{Proposition}
\newtheorem{lemma}[theorem]{Lemma}
\newtheorem{corollary}[theorem]{Corollary}
\newtheorem{assumption}[theorem]{Assumption}
\theoremstyle{definition}
\newtheorem{definition}[theorem]{Definition}
\theoremstyle{remark}

\newcommand{\E}{\mathbb{E}}
\newcommand{\R}{\mathbb{R}}

\newcommand{\calS}{\mathcal{S}}
\newcommand{\calZ}{\mathcal{Z}}
\newcommand{\calV}{\mathcal{V}}

\newcommand{\Reg}{\mathrm{Reg}}

\newcommand{\argmin}{\operatorname*{arg\,min}}

\title{Decision-Weighted Flow Matching for Contextual Stochastic Optimization}

\author{
Jize Xie\thanks{Equal contribution.} \\
  Department of Industrial Engineering and Decision Analytics\\
  Hong Kong University of Science and Technology \\
  \texttt{jxiebj@connect.ust.hk} \\
\And
Haomiao Wu\footnotemark[1] \\
  Big Data Institute\\
  Central South University \\
  \texttt{wuhaomiao@csu.edu.cn}\\
\AND
Qiang Chen\\
  Department of Industrial Engineering and Decision Analytics\\
  Hong Kong University of Science and Technology \\
  \texttt{qchencw@connect.ust.hk}
\And
Xiu Su\thanks{Corresponding authors.} \\
  Big Data Institute\\
  Central South University \\
  \texttt{xiusu1994@csu.edu.cn}
\AND
Yi Chen\footnotemark[2] \\
  Department of Industrial Engineering and Decision Analytics\\
  Hong Kong University of Science and Technology \\
  \texttt{yichen@ust.hk}
}

\date{}
\makeatletter
\renewcommand{\@notice}{}
\makeatother

\begin{document}
\maketitle
\thispagestyle{empty}
\pagestyle{empty}
\begin{abstract}
Conditional generative models are increasingly used as scenario generators for
stochastic optimization, but standard training objectives emphasize uniform
distributional fit rather than the downstream decisions induced by generated
scenarios. This creates an objective mismatch: errors in statistically common
regions may have little effect on regret, whereas errors in decision-sensitive
regions can substantially change the optimal action. We propose
Decision-Weighted Flow Matching (DW-FM), a regret-aligned training framework
that preserves the simplicity of standard flow matching while reweighting its
velocity-regression objective using decision-sensitive endpoint information.
Theoretically, we connect downstream regret to pathwise velocity mismatch
through a loss-induced decision discrepancy and an adjoint transport argument,
yielding an ideal regret-aligned surrogate and practical endpoint-weighted
objectives with regret guarantees. Empirically, we demonstrate the effectiveness of DW-FM on three CVaR-based contextual stochastic
optimization benchmarks spanning synthetic portfolio, semi-real financial, and
traffic-CVaR tasks, where DW-FM improves downstream regret over standard baselines.
\end{abstract}

\section{Introduction}

Conditional generative models are a natural interface between prediction and
stochastic optimization. Given a context, a model generates
scenarios for uncertain outcomes, and a downstream solver chooses a decision
from the induced distribution. This distributional view is important in
applications such as inventory control, portfolio allocation, and risk-sensitive
planning, where decisions may depend on tail events or multimodal uncertainty
rather than on mean prediction alone~\citep{shapiro2009lectures,
rockafellar2000optimization,wang2025gendfl,zhao2025diffusiondfl}. However, in such pipelines the learned conditional law is only an intermediate object. The final objective is not distribution matching itself, but low downstream regret. 

This creates an objective mismatch. Standard generative objectives allocate
training effort according to distributional fit: errors are weighted by how often
the corresponding regions are sampled by the training objective. Stochastic
optimization weights errors differently. An error in a decision-sensitive tail
region can substantially change the optimizer, while an error in a common but
decision-irrelevant region may have little effect on the final decision. Thus, a
scenario generator can fit the conditional distribution well on average but still
make the errors that matter most for regret.

While decision-focused learning addresses a related mismatch for predictive models by
training them through the decisions they induce rather than through prediction
loss alone~\citep{donti2017task,elmachtoub2022smart,wilder2019melding,
ferber2020mipaal}, much of this literature treats the learned object as a point prediction,
deterministic surrogate, or differentiable optimization layer. Recent generative decision-learning methods move to generative or diffusion-based scenario generators, often
combined with end-to-end decision-gradient training
~\citep{wang2025gendfl,zhao2025diffusiondfl}. Our focus is more targeted:
\textit{given a conditional generator, can its training loss be
made regret-aligned so that fitting effort is concentrated on decision-relevant
regions?}

We study this question in the context of flow matching. Conditional flow matching trains a
scenario generator by local velocity regression along interpolation
paths~\citep{lipman2023flow}. This structure is attractive because the
regression loss can be reweighted without changing the architecture, velocity
labels, sampling procedure, or downstream solver. The key challenge is to choose weights that are
principled for regret: downstream regret is a terminal, optimizer-level quantity,
whereas flow matching is trained through local pathwise velocity errors.

We propose \textbf{Decision-Weighted Flow Matching} (DW-FM), a regret-aligned
training framework for contextual stochastic optimization. DW-FM first connects
regret to an ideal adjoint-weighted pathwise velocity error, and then replaces
this ideal but generally intractable weight with a computable endpoint
sensitivity score. The resulting algorithm is a plug-in modification of
standard conditional flow matching: it simply multiplies each per-sample FM
regression loss by a decision-sensitive endpoint weight.
Our theory justifies this reweighting as a regret-aligned surrogate. We show
that downstream regret is controlled by a loss-induced decision discrepancy, and
then use an adjoint transport argument to relate this discrepancy to weighted
pathwise velocity error. Finally, we prove
a regret bound that separates the trainable weighted excess risk from the bias introduced by reweighting, with a finite-sample extension. Empirically, DW-FM reduces downstream regret on synthetic portfolio, semi-real
financial, and PEMS-BAY traffic-CVaR benchmarks, outperforming standard
baselines in the main comparisons and improving downside-tail diagnostics.

To sum up, our contributions are as follows:
\begin{itemize}
    \item We formulate decision-aligned conditional scenario generation for
    contextual stochastic optimization through a loss-induced decision
    discrepancy that directly bounds regret.

    \item We propose {Decision-Weighted Flow Matching} (DW-FM), a
    plug-in endpoint-weighted objective for conditional flow matching that
    changes only the per-sample regression weights in standard FM regression.

    \item 
    We provide theoretical guarantees for DW-FM. We derive an ideal
    adjoint-weighted pathwise surrogate, characterize the population target
    induced by endpoint weighting, and prove a regret bound in terms of
    weighted excess risk and tilting bias, with a finite-sample
    extension.

    \item Extensive experiments on three CVaR-based benchmarks show that DW-FM
reduces downstream regret relative to Uniform FM, improves over standard
predict-then-optimize and decision-learning baselines in the main comparisons,
and yields better downside-tail diagnostics.
\end{itemize}

\section{Related Work}

\noindent\textbf{Decision-Focused Learning and Stochastic Optimization.}
Decision-focused learning and stochastic optimization argue that predictive models should be evaluated by the decisions they induce, not only by prediction error. Classical stochastic programming provides the optimization foundation for decision-making under uncertainty~\citep{shapiro2009lectures, bertsimas2020predictive, kallus2023stochastic}, while task-based end-to-end learning and SPO/SPO+ explicitly train predictors for downstream decision quality~\citep{donti2017task,elmachtoub2022smart}. Subsequent work extends this principle to combinatorial and differentiable optimization layers~\citep{amos2017optnet,wilder2019melding,ferber2020mipaal,agrawal2019differentiable,berthet2020learning,mandi2024decision}. However, much of this line focuses on point predictions or deterministic surrogates passed to a solver. Our setting is different: the learned object is a full conditional law used as a scenario generator. We therefore ask how regret should reshape the surrogate used to train the conditional generator itself.

\noindent\textbf{Conditional Generative Models for Decision Making.}
Generative models are increasingly used to represent conditional uncertainty
through samples, including probabilistic forecasting and diffusion-based
time-series models~\citep{rasul2021autoregressive,yan2021scoregrad}.
Recent
generative decision-learning methods bring this idea into downstream optimization:
Gen-DFL uses generative modeling to capture uncertainty and improve robust
decision quality, while Diffusion-DFL trains diffusion predictors for stochastic
optimization using reparameterization or score-function estimators
~\citep{wang2025gendfl,zhao2025diffusiondfl}. These works establish the value of
distributional scenario generation for decision making. Our focus is
complementary: we keep the flow-matching backbone fixed and redesign the
generative surrogate so that fitting effort is allocated according to downstream
regret sensitivity.

\noindent\textbf{Flow Matching and Weighted Surrogate Design.}
Flow matching provides a simulation-free framework for continuous-time
generative modeling by reducing training to velocity-field regression along
probability paths~\citep{lipman2023flow}. Related interpolation-based and
straight-path generative modeling frameworks include rectified flow and
stochastic interpolants~\citep{liu2022flow,albergo2022building, albergo2025stochastic}, while conditional and optimal-transport flow matching
further improve training stability and path design for continuous normalizing
flows~\citep{tong2024improving}. This pathwise regression form makes flow
matching a natural place to introduce decision-sensitive weighting.
Related cost-sensitive and importance-weighted learning methods also reweight
samples according to task relevance, class imbalance, or distribution shift
~\citep{elkan2001foundations,zadrozny2003cost,shimodaira2000improving,
sugiyama2007covariate,byrd2019effect}. In contrast to generic reweighting,
DW-FM derives its weight from regret control in conditional stochastic
optimization, yielding an adjoint-weighted ideal surrogate and a practical
endpoint-weighted plug-in objective that changes only the per-sample FM loss.

\section{Problem Formulation}
\label{sec:problem_formulation}
We study a predict-then-optimize pipeline for contextual stochastic optimization. In the prediction stage, given an observed context
\(\mathbf{x}\in\mathcal{X}\), the learning system produces a conditional
scenario generator for the uncertain quantity
\(\mathbf{S}\in\mathcal{S}\subseteq\mathbb{R}^d\). This generator induces a
conditional distribution and is accessed through generated scenarios. In the
optimization stage, the induced conditional
distribution is passed to a downstream stochastic optimization solver, which
chooses a decision \(\mathbf{z}\in\mathcal{Z}\subseteq\mathbb{R}^m\) to minimize
expected scenario loss.
Let $q_{\mathbf{x}}^\star$ denote the true conditional law of $\mathbf{S}$ given
$\mathbf{X}=\mathbf{x}$. For any candidate probability law $q$ on $\mathcal{S}$, define the
context-specific risk
$R_{\mathbf{x}}(\mathbf{z};q)
:=
\mathbb{E}_{\mathbf{S}\sim q}
\!\left[
\ell_{\mathbf{x}}(\mathbf{z},\mathbf{S})
\right].$
If the true law were known, the optimal decision would be
$\mathbf{z}_{\mathbf{x}}^\star
\in
\argmin_{\mathbf{z}\in\mathcal{Z}}
R_{\mathbf{x}}(\mathbf{z};q_{\mathbf{x}}^\star).$

In practice, $q_{\mathbf{x}}^\star$ is usually unknown, but we observe an i.i.d.\ training dataset $\mathcal{D} = \{(\mathbf{x}_i, \mathbf{s}_i)\}_{i=1}^n$ to train a prediction model, where the pairs are independent realizations of \((\mathbf X,\mathbf S)\). Conditional generative models offer a natural modeling framework for this task. By learning an expressive conditional law $q_{\theta,\mathbf{x}}$, parameterized by trainable
parameters $\theta$, the model serves as a data-driven scenario generator. 
Given the learned generative surrogate, the downstream solver returns the plug-in decision $\mathbf{z}_{\theta}(\mathbf{x}) \in \argmin_{\mathbf{z} \in \mathcal{Z}} R_{\mathbf{x}}(\mathbf{z}; q_{\theta,\mathbf{x}})$. We evaluate the utility of this model by the {decision regret} it incurs under the true conditional law:
\[
\mathrm{Reg}_{\mathbf{x}}(\theta)
:=
R_{\mathbf{x}}\!\left(\mathbf{z}_{\theta}(\mathbf{x}); q_{\mathbf{x}}^\star\right)
-
R_{\mathbf{x}}\!\left(\mathbf{z}_{\mathbf{x}}^\star; q_{\mathbf{x}}^\star\right).
\]
Our objective is to minimize the expected regret, $\mathbb E\!\left[\mathrm{Reg}_{\mathbf{X}}(\theta)\right]$. 
This criterion highlights the mismatch between ordinary distribution matching
and decision quality. A generative model may approximate the true conditional
distribution well in a global statistical sense while still making errors in
regions that strongly affect the optimizer. We
therefore need a metric that measures distributional error through the
downstream loss class rather than through a generic distributional distance.

\begin{definition}[Decision discrepancy]
For any probability laws $q$ and $q'$ on $\mathcal{S}$, define
\[
d_{\mathrm{dec},\mathbf{x}}(q,q')
:=
\sup_{\mathbf{z} \in \mathcal{Z}}
\left|
R_{\mathbf{x}}(\mathbf{z}; q) - R_{\mathbf{x}}(\mathbf{z}; q')
\right|.
\]
\end{definition}
 Unlike generic distributional distances, $d_{\mathrm{dec},\mathbf{x}}$
 only measures discrepancies that can change downstream risks. Crucially, controlling this discrepancy is sufficient to upper-bound the regret:

\begin{proposition}
\label{prop:regret-discrepancy}
For every context $\mathbf{x}$ and every learned law $q_{\theta,\mathbf{x}}$ on $\mathcal{S}$, we have $\mathrm{Reg}_{\mathbf{x}}(\theta) \le 2\, d_{\mathrm{dec},\mathbf{x}}(q_{\mathbf{x}}^\star, q_{\theta,\mathbf{x}})$. Consequently, the expected regret satisfies
$\mathbb E\!\left[\mathrm{Reg}_{\mathbf{X}}(\theta)\right]
\le
2\,\mathbb E\!\left[
d_{\mathrm{dec},\mathbf{X}}(q_{\mathbf{X}}^\star, q_{\theta,\mathbf{X}})
\right].$
\end{proposition}

Proposition~\ref{prop:regret-discrepancy} reduces the original decision problem to a surrogate-design task: construct a trainable generative objective that minimizes the decision discrepancy between $q_{\theta,\mathbf{x}}$ and $q_{\mathbf{x}}^\star$. 
\section{Decision-Weighted Flow Matching}
\label{sec:method}
\subsection{Flow matching for conditional generation}
We use flow matching as the conditional generative backbone. The model is a parameterized vector field
$\mathbf{v}_\theta:[0,1]\times\mathcal{S}\times\mathcal{X}\to\mathbb{R}^d,$
where $\theta$ collects the trainable parameters. For each context $\mathbf{x}$, let
$q_{0,\mathbf{x}}$ be a simple base distribution on $\mathcal{S}$. Starting
from $\mathbf{S}_0\sim q_{0,\mathbf{x}}$, the learned flow evolves according to
$\frac{d\mathbf{S}_t}{dt}
=
\mathbf{v}_\theta(t,\mathbf{S}_t,\mathbf{x}),$ for
$t\in[0,1].$
We write $q_{\mathbf{x},t}^{\theta}$ for the marginal law of $\mathbf{S}_t$, and
define the learned terminal law by
$q_{\theta,\mathbf{x}}:=q_{\mathbf{x},1}^{\theta}.$

Flow matching trains $\mathbf{v}_\theta$ by supervised regression. In population
form, draw a data endpoint \((\mathbf X,\mathbf S_1)\), so that conditionally on
\(\mathbf X=\mathbf x\), the endpoint \(\mathbf S_1\) follows the true
conditional distribution \(q_{\mathbf x}^{\star}\). Independently draw
$\mathbf S_0\sim q_{0,\mathbf X},
T\sim\mathrm{Unif}[0,1],$
and construct the linear interpolation
\[
\mathbf S_T=(1-T)\mathbf S_0+T\mathbf S_1,
\qquad
\boldsymbol\Delta:=\mathbf S_1-\mathbf S_0,
\qquad
\mathbf Y:=(T,\mathbf S_T,\mathbf X).
\]
Here \(\mathbf S_0\) is the base sample, \(\mathbf S_1\) is the data endpoint,
and \(\boldsymbol\Delta\) is the constant velocity along the straight-line path
connecting them. The tuple \(\mathbf Y\) is the regression input observed by the model.

The standard conditional flow-matching objective is
\[
L_{\mathrm{FM}}(\theta)
=
\mathbb{E}\!\left[
\left\|
\mathbf{v}_\theta(T,\mathbf{S}_T,\mathbf{X})
-
\boldsymbol{\Delta}
\right\|^2
\right].
\]
That is, at a random interpolation point $(T,\mathbf{S}_T)$ and context $\mathbf{X}$, the model is
trained to predict the velocity pointing from the base sample to the data endpoint. Thus,
FM reduces conditional generation to supervised vector-field regression. However, this objective is
agnostic to downstream decision. We aim to modify this loss so that it emphasizes the
parts that matter most for decision making.

\subsection{From Decision Discrepancy to Endpoint Weights}
\label{subsec:decision_sensitive_weighting}

FM penalizes all squared velocity errors uniformly. In a
predict-then-optimize pipeline, however, the relevant error is the decision discrepancy
\(d_{\mathrm{dec},\mathbf{x}}\). We
therefore seek a weighting mechanism that gives larger weight to velocity
errors that can induce larger changes in downstream risks.

Fix a context \(\mathbf{x}\), and let
\(\mathbf u_{\mathbf{x}}(t,\mathbf{s})\) be an ideal target velocity field
transporting \(q_{0,\mathbf{x}}\) to \(q_{\mathbf{x}}^\star\). For a fixed
decision \(\mathbf z\), the terminal risk error is driven by how local velocity
errors along the path perturb the terminal loss
\(\ell_{\mathbf{x}}(\mathbf z,\cdot)\). To express this sensitivity, define the
backward transported loss \(\phi_{\mathbf{x},\mathbf z}\) by
\begin{equation}
\partial_t\phi_{\mathbf{x},\mathbf z}(t,\mathbf{s})
+
\mathbf u_{\mathbf{x}}(t,\mathbf{s})^\top
\nabla_{\mathbf{s}}\phi_{\mathbf{x},\mathbf z}(t,\mathbf{s})
=
0,
\qquad
\phi_{\mathbf{x},\mathbf z}(1,\mathbf{s})
=
\ell_{\mathbf{x}}(\mathbf z,\mathbf{s}).
\label{eq:adjoint-def-method}
\end{equation}
Thus \(\phi_{\mathbf{x},\mathbf z}(t,\mathbf{s})\) is the terminal loss pulled
back to the path location \((t,\mathbf{s})\), and
\(\nabla_{\mathbf{s}}\phi_{\mathbf{x},\mathbf z}(t,\mathbf{s})\) measures how a
local perturbation at that location changes the terminal risk of decision
\(\mathbf z\). Therefore, a velocity error at \((t,\mathbf{s})\) should receive
larger training weight when this gradient is large, because the same local
transport error can then induce a larger error in the downstream risk.
Since \(d_{\mathrm{dec},\mathbf{x}}\) takes the worst case over downstream
decisions, the corresponding ideal pathwise sensitivity envelope is
$M_{\mathbf{x}}(t,\mathbf{s})
:=
\sup_{\mathbf z\in\mathcal Z}
\left\|
\nabla_{\mathbf{s}}\phi_{\mathbf{x},\mathbf z}(t,\mathbf{s})
\right\|^2 .$
This envelope is the pathwise weight induced by the decision discrepancy:
locations with large \(M_{\mathbf{x}}(t,\mathbf{s})\) are precisely those where
small velocity errors can produce large changes in some downstream risk. In
Section~\ref{sec:theory}, we show that the resulting \(M_{\mathbf{x}}\)-weighted
pathwise velocity error controls \(d_{\mathrm{dec},\mathbf{x}}\), and hence
regret.

However, \(M_{\mathbf{x}}\) is not directly trainable. We therefore propose a computable endpoint score that preserves
the main decision-sensitivity signal. At terminal time,
$\nabla_{\mathbf{s}}\phi_{\mathbf{x},\mathbf z}(1,\mathbf{s})
=
\nabla_{\mathbf{s}}\ell_{\mathbf{x}}(\mathbf z,\mathbf{s}),$
so terminal loss gradients provide a tractable proxy for adjoint sensitivity.
Moreover, rather than taking a global supremum over all decisions, we use the
decision around which regret is locally determined. If the oracle decision were
known, this gives the oracle endpoint score
$w_{\mathbf{x}}^\star(\mathbf{s})
=
1
+
\lambda
\left\|
\nabla_{\mathbf{s}}
\ell_{\mathbf{x}}(\mathbf z_{\mathbf{x}}^\star,\mathbf{s})
\right\|^2 .$
The constant term retains ordinary FM coverage, while \(\lambda\ge 0\) controls
the strength of decision-sensitive reweighting.
Since \(\mathbf z_{\mathbf{x}}^\star\) depends on the unknown true conditional
distribution, we replace it with a reference decision
\(\widehat{\mathbf z}_{\mathbf{x}}\) and use the plug-in endpoint score
\[
\widehat w_{\mathbf{x}}(\mathbf{s})
=
1
+
\lambda
\left\|
\nabla_{\mathbf{s}}
\ell_{\mathbf{x}}(\widehat{\mathbf z}_{\mathbf{x}},\mathbf{s})
\right\|^2 .
\]
$\widehat{\mathbf z}_{\mathbf{x}}$ can be obtained from a baseline predictor, or a sample-average approximation solution. The quality of this plug-in score depends on the accuracy of \(\widehat{\mathbf z}_{\mathbf{x}}\). We precisely quantify their relationship in Appendix~\ref{app:plugin-weight-stability}.

\subsection{Decision-Weighted Flow Matching Objective and Plug-in Decision}
\label{subsec:practical_DW-FM}
We now use the score \(\widehat w_{\mathbf{x}}(\mathbf{s})\) to define the
trainable method, which we call \emph{Decision-Weighted Flow Matching}
(DW-FM). Its population objective is
\[L_{\mathrm{DW-FM}}(\theta):=\mathbb E[\widehat w_{\mathbf X}(\mathbf S_1)
\|\mathbf v_\theta(T,\mathbf S_T,\mathbf X)-\boldsymbol\Delta\|^2].\] 

This objective preserves the standard FM regression label
\(\boldsymbol\Delta\); it only changes how much each endpoint contributes to
the squared velocity-regression loss. In empirical training, DW-FM is implemented by
multiplying the usual per-sample FM loss by the decision-sensitive weight of
the observed endpoint, with the reference decision
\(\widehat{\mathbf z}_{\mathbf x}\) treated as fixed. The minibatch procedure is summarized in Algorithm~\ref{alg:DW-FM} in
Appendix~\ref{app:A}.

At test time, for a new context \(\mathbf x\), we generate scenarios by drawing
\(\mathbf S_{0,k}\sim q_{0,\mathbf x}\) and integrating
\(d\mathbf S_t/dt=\mathbf v_\theta(t,\mathbf S_t,\mathbf x)\) to obtain
\(\widetilde{\mathbf s}_1,\ldots,\widetilde{\mathbf s}_K\sim
q_{\theta,\mathbf x}\). The downstream decision is then computed by the
sample-average plug-in problem
\(\widehat{\mathbf z}_\theta(\mathbf x)\in\argmin_{\mathbf z\in\mathcal Z}
K^{-1}\sum_{k=1}^K\ell_{\mathbf x}(\mathbf z,\widetilde{\mathbf s}_k)\).
\section{Theoretical Analysis}
\label{sec:theory}
\subsection{Ideal Pathwise Regret Control}
\label{subsec:ideal_pathwise_regret}
We first introduce an ideal pathwise surrogate induced by the decision
discrepancy. Fix a context \(\mathbf{x}\). At this stage,
\(\mathbf u_{\mathbf{x}}\) denotes a sufficiently regular target velocity field
transporting \(q_{0,\mathbf{x}}\) to \(q_{\mathbf{x}}^\star\); in
Sections~\ref{subsec:endpoint_population_target}--\ref{subsec:DW-FM_regret}, we
specialize this target path to the FM interpolation path.

\begin{assumption}[Path and adjoint regularity]
\label{ass:path-adjoint-regularity}
Fix a context \(\mathbf{x}\). There exists a target path
\(\{\mu_{\mathbf{x},t}\}_{t\in[0,1]}\) transporting \(q_{0,\mathbf{x}}\) to
\(q_{\mathbf{x}}^\star\), and the learned path
\(\{q_{\mathbf{x},t}^{\theta}\}_{t\in[0,1]}\) is induced by
\(\mathbf v_\theta(\cdot,\cdot,\mathbf{x})\) with endpoints
\(q_{0,\mathbf{x}}\) and \(q_{\theta,\mathbf{x}}\). Both paths admit densities,
$\mu_{\mathbf{x},t}(d\mathbf{s})
=
\rho_{\mathbf{x},t}(\mathbf{s})\,d\mathbf{s},
q_{\mathbf{x},t}^{\theta}(d\mathbf{s})
=
\rho_{\mathbf{x},t}^{\theta}(\mathbf{s})\,d\mathbf{s},$
which are sufficiently regular and satisfy the continuity equations
driven by \(\mathbf u_{\mathbf{x}}\) and
\(\mathbf v_\theta(\cdot,\cdot,\mathbf{x})\), respectively. Moreover, for every
\(\mathbf z\in\mathcal Z\), the loss
\(\ell_{\mathbf{x}}(\mathbf z,\cdot)\) is differentiable in \(\mathbf s\), the transport equation~\eqref{eq:adjoint-def-method} admits a solution
\(\phi_{\mathbf{x},\mathbf z}\) with integrable gradient, and all risks and
pathwise integrals used below are finite.
\end{assumption}
Assumption~\ref{ass:path-adjoint-regularity} collects the smoothness and
integrability conditions needed to differentiate transported risks and apply the
adjoint identity.
\begin{assumption}[No-boundary-flux condition]
\label{ass:no-boundary-flux}
Let \(\partial\mathcal S\) denote the boundary of \(\mathcal S\). For every
\(t\in[0,1]\) and \(\mathbf z\in\mathcal Z\), the boundary flux terms vanish:
\[
\begin{aligned}
\int_{\partial\mathcal S}
\phi_{\mathbf{x},\mathbf z}(t,\mathbf s)
\rho_{\mathbf{x},t}^{\theta}(\mathbf s)
\mathbf v_\theta(t,\mathbf s,\mathbf x)^\top
\mathbf n(\mathbf s)\,dA(\mathbf s)
=\int_{\partial\mathcal S}
\phi_{\mathbf{x},\mathbf z}(t,\mathbf s)
\rho_{\mathbf{x},t}(\mathbf s)
\mathbf u_{\mathbf{x}}(t,\mathbf s)^\top
\mathbf n(\mathbf s)\,dA(\mathbf s)=0,
\end{aligned}
\]
where \(\mathbf n(\mathbf s)\) is the outward unit normal and \(dA\) is surface
measure.
\end{assumption}
This condition rules out changes in transported expected loss caused by
probability mass entering or leaving \(\mathcal S\) through the boundary. It is
standard in transport and continuous-adjoint analyses
\citep{santambrogio2015optimal,cances2020variational}; for example, it holds
under periodic boundaries, zero normal probability flux on compact domains, or
sufficient decay on \(\mathbb R^d\).
Using the ideal sensitivity envelope \(M_{\mathbf{x}}\) from
Section~\ref{subsec:decision_sensitive_weighting}, define
\[
L_{\mathrm{ideal},\mathbf{x}}(\theta)
:=
\int_0^1
\!\int_{\mathcal S}
M_{\mathbf{x}}(t,\mathbf s)
\left\|
\mathbf v_\theta(t,\mathbf s,\mathbf x)
-
\mathbf u_{\mathbf{x}}(t,\mathbf s)
\right\|^2
\,dq_{\mathbf{x},t}^{\theta}(\mathbf s)\,dt .
\]

\begin{theorem}
\label{thm:ideal-controls-regret}
Under Assumptions~\ref{ass:path-adjoint-regularity}
and~\ref{ass:no-boundary-flux}, for every fixed context \(\mathbf{x}\),
$d_{\mathrm{dec},\mathbf{x}}(q_{\mathbf{x}}^\star,q_{\theta,\mathbf{x}})
\le
L_{\mathrm{ideal},\mathbf{x}}(\theta)^{1/2}.$
Consequently,
$\mathbb E
\!\left[
\mathrm{Reg}_{\mathbf X}(\theta)
\right]
\le
2\,
\mathbb E
\!\left[
L_{\mathrm{ideal},\mathbf X}(\theta)^{1/2}
\right].$
\end{theorem}

Theorem~\ref{thm:ideal-controls-regret} directly shows that pathwise velocity errors weighted by
\(M_{\mathbf{x}}(t,\mathbf s)\) control the decision discrepancy, and therefore
regret. Thus \(L_{\mathrm{ideal},\mathbf{x}}\) is the ideal benchmark
that motivates DW-FM. When the downstream risk is strongly convex, Appendix~\ref{app:strongly-convex-first-order}
gives a sharper first-order closure that replaces the square-root conversion in
Theorem~\ref{thm:ideal-controls-regret} by a linear regret conversion.
\subsection{Population Target of the Endpoint-Weighted Objective}
\label{subsec:endpoint_population_target}

The ideal bound in Section~\ref{subsec:ideal_pathwise_regret} controls regret
through the error between the learned field \(\mathbf v_\theta\) and a target
velocity \(\mathbf u_{\mathbf{x}}\). To connect this bound to DW-FM, we now
specialize the target path to the FM interpolation path. Under this choice,
\(\mathbf u_{\mathbf{x}}\) is the ordinary FM population target. DW-FM, however,
optimizes an endpoint-weighted regression objective. This subsection identifies
the population target selected by that objective and quantifies its deviation
from \(\mathbf u_{\mathbf{x}}\).

Fix a context \(\mathbf{x}\). Write
\(\mathbb E_{\mathbf{x}}[\cdot]:=\mathbb E[\cdot\mid \mathbf X=\mathbf{x}]\),
where \((\mathbf S_1,\mathbf S_0,T)\) follows the conditional interpolation law
from Section~\ref{subsec:practical_DW-FM}. The ordinary FM population target is
$\mathbf u_{\mathbf{x}}(\mathbf Y)
:=
\mathbb E_{\mathbf{x}}[\boldsymbol\Delta\mid \mathbf Y].$
Let \(w_{\mathbf{x}}:\mathcal S\to[0,\infty)\) be a fixed endpoint-weight
function. This includes the oracle endpoint score \(w_{\mathbf{x}}^\star\), and,
once the reference decision is fixed, the plug-in score
\(\widehat w_{\mathbf{x}}\). Define the fixed-context endpoint-weighted
population loss
\[
L_{w,\mathbf{x}}(\mathbf v)
:=
\mathbb E_{\mathbf{x}}
\!\left[
w_{\mathbf{x}}(\mathbf S_1)
\left\|
\mathbf v(\mathbf Y)-\boldsymbol\Delta
\right\|^2
\right].
\]
Thus \(L_{w,\mathbf{x}}\) is the conditional population version of the
trainable DW-FM objective; taking \(w_{\mathbf{x}}=\widehat w_{\mathbf{x}}\)
recovers the objective in Section~\ref{subsec:practical_DW-FM}.
Because \(w_{\mathbf{x}}(\mathbf S_1)\) remains random after conditioning on
the regression input \(\mathbf Y\), define the effective conditional weight
$m_{w,\mathbf{x}}(\mathbf Y)
:=
\mathbb E_{\mathbf{x}}
[
w_{\mathbf{x}}(\mathbf S_1)\mid \mathbf Y
],$
and, on \(\{m_{w,\mathbf{x}}(\mathbf Y)>0\}\), define
\[
\mathbf u_{w,\mathbf{x}}(\mathbf Y)
:=
\frac{
\mathbb E_{\mathbf{x}}
[
w_{\mathbf{x}}(\mathbf S_1)\boldsymbol\Delta\mid \mathbf Y
]
}{
m_{w,\mathbf{x}}(\mathbf Y)
}.
\]
\begin{theorem}
\label{thm:endpoint-population-target}
The function \(\mathbf u_{w,\mathbf{x}}\) is a population minimizer of
\(L_{w,\mathbf{x}}\). Moreover, for every measurable vector field \(\mathbf v\),
\[
L_{w,\mathbf{x}}(\mathbf v)-L_{w,\mathbf{x}}(\mathbf u_{w,\mathbf{x}})
=
\mathbb E_{\mathbf{x}}
\!\left[
m_{w,\mathbf{x}}(\mathbf Y)
\left\|
\mathbf v(\mathbf Y)-\mathbf u_{w,\mathbf{x}}(\mathbf Y)
\right\|^2
\right].
\]
In addition, on \(\{m_{w,\mathbf{x}}(\mathbf Y)>0\}\),
$\mathbf u_{w,\mathbf{x}}(\mathbf Y)
-
\mathbf u_{\mathbf{x}}(\mathbf Y)
=
{
\operatorname{Cov}_{\mathbf{x}}
\!\left(
w_{\mathbf{x}}(\mathbf S_1),
\boldsymbol\Delta
\mid
\mathbf Y
\right)
}/{
m_{w,\mathbf{x}}(\mathbf Y)
},$
where the scalar--vector covariance is understood componentwise.
\end{theorem}

Theorem~\ref{thm:endpoint-population-target} gives the key bridge from the
trainable objective to the regret analysis. The excess endpoint-weighted loss
controls the distance from \(\mathbf v_\theta\) to the weighted population target
\(\mathbf u_{w,\mathbf{x}}\), while the covariance identity quantifies the
target shift from the ordinary FM velocity \(\mathbf u_{\mathbf{x}}\). This
target shift is the tilting term that appears in the regret bound below.
\subsection{Population Regret for DW-FM}
\label{subsec:DW-FM_regret}

We now combine Sections~\ref{subsec:ideal_pathwise_regret}
and~\ref{subsec:endpoint_population_target} to obtain a population regret
closure for DW-FM. The ideal bound controls regret through the pathwise error
between \(\mathbf v_\theta\) and the ordinary FM target velocity
\(\mathbf u_{\mathbf{x}}\), while the trainable DW-FM objective controls the
error between \(\mathbf v_\theta\) and the endpoint-weighted population target
\(\mathbf u_{\widehat w,\mathbf{x}}\). Thus the key decomposition is
$\mathbf v_\theta-\mathbf u_{\mathbf{x}}
=
(\mathbf v_\theta-\mathbf u_{\widehat w,\mathbf{x}})
+
(\mathbf u_{\widehat w,\mathbf{x}}-\mathbf u_{\mathbf{x}}).$

Using the plug-in endpoint score \(\widehat w_{\mathbf{x}}\) from
Section~\ref{subsec:decision_sensitive_weighting}, define
\[
E_{\widehat w,\mathbf{x}}(\mathbf v_\theta)
:=
L_{\widehat w,\mathbf{x}}(\mathbf v_\theta)
-
L_{\widehat w,\mathbf{x}}(\mathbf u_{\widehat w,\mathbf{x}}),\quad
B_{\mathrm{tilt},\mathbf{x}}(\widehat w)
:=
\mathbb E_{\mathbf{x}}
\!\left[
m_{\widehat w,\mathbf{x}}(\mathbf Y)
\left\|
\mathbf u_{\widehat w,\mathbf{x}}(\mathbf Y)
-
\mathbf u_{\mathbf{x}}(\mathbf Y)
\right\|^2
\right].
\]
Here \(E_{\widehat w,\mathbf{x}}(\mathbf v_\theta)\) measures how well the
learned field fits the endpoint-weighted population target, while
\(B_{\mathrm{tilt},\mathbf{x}}(\widehat w)\) measures the target shift identified
in Section~\ref{subsec:endpoint_population_target}.
To transfer the ideal regret bound to this trainable objective, we require the
endpoint-weighted interpolation distribution to cover the path regions that are
important under the ideal regret weight.

\begin{assumption}[Path-overlap and sensitivity coverage]
\label{ass:DW-FM-coverage}
Fix the vector field \(\mathbf v_\theta\) and a context \(\mathbf{x}\). Let
\(p_{\mathbf{x}}(t,\mathbf{s})\) be the density of \((T,\mathbf S_T)\) under the
conditional FM interpolation law. There exist finite constants
\(A_{\mathbf{x}}\) and \(B_{\mathbf{x}}\) such that, for almost every
\((t,\mathbf{s})\in[0,1]\times\mathcal S\),
\[
\rho_{\mathbf{x},t}^{\theta}(\mathbf{s})
\le
A_{\mathbf{x}}p_{\mathbf{x}}(t,\mathbf{s}),
\qquad
M_{\mathbf{x}}(t,\mathbf{s})
\le
B_{\mathbf{x}}m_{\widehat w,\mathbf{x}}(t,\mathbf{s}).
\]
\end{assumption}

The first inequality is a path-overlap condition: the learned ODE path should
remain covered by the FM interpolation path. The second inequality is a
sensitivity-coverage condition: the endpoint-induced path weight should be
large on regions with high ideal decision sensitivity. Together they imply
\[
M_{\mathbf{x}}(t,\mathbf{s})\rho_{\mathbf{x},t}^{\theta}(\mathbf{s})
\le
C_{\widehat w,\mathbf{x}}
m_{\widehat w,\mathbf{x}}(t,\mathbf{s})
p_{\mathbf{x}}(t,\mathbf{s}),
\qquad
C_{\widehat w,\mathbf{x}}:=A_{\mathbf{x}}B_{\mathbf{x}}.
\]
\begin{theorem}
\label{thm:DW-FM-regret}
Under the conditions of Theorem~\ref{thm:ideal-controls-regret} and
Assumption~\ref{ass:DW-FM-coverage},
\[
\mathbb E
\!\left[
\mathrm{Reg}_{\mathbf X}(\theta)
\right]
\le
2
\mathbb E
\!\left[
\sqrt{
2C_{\widehat w,\mathbf X}
\left(
E_{\widehat w,\mathbf X}(\mathbf v_\theta)
+
B_{\mathrm{tilt},\mathbf X}(\widehat w)
\right)
}
\right].
\]
\end{theorem}

Theorem~\ref{thm:DW-FM-regret} shows that DW-FM regret is governed by two
quantities: the endpoint-weighted excess risk and the tilting bias. The former
is the error controlled by the trainable objective; the latter measures the
population target shift caused by endpoint weighting. Thus, expected regret is small whenever the learned
field fits the endpoint-weighted population target well and the induced target
shift is controlled.
Setting \(\lambda=0\), DW-FM
reduces to ordinary FM and the tilting bias vanishes. For \(\lambda>0\),
decision weighting can redirect approximation capacity toward
decision-sensitive path regions, but it also introduces a tilting term. The
effect of decision weighting is therefore a bias--fit tradeoff rather than an
unconditional improvement. Appendix~\ref{app: misspecification} further illustrates the benefit of decision weighting. We also provide a corresponding finite-sample analysis in Appendix \ref{app:finite-sample}.

\section{Experiments}
\label{sec:experiment}

\subsection{Experimental Setup}
We evaluate DW-FM on three CVaR-based contextual stochastic optimization
benchmarks: a controlled synthetic portfolio sweep, a semi-real financial
portfolio task based on Ken French industry portfolios and Fama--French
factors~\citep{kenfrench,fama1993common}, and a PEMS-BAY traffic congestion
task~\citep{pemsbay}. These benchmarks share the same evaluation structure:
given a context $\mathbf{x}$, the learning method produces scenarios or
predictions for an uncertain outcome vector, and the downstream solver computes
a feasible decision under a mean-loss plus CVaR objective. We use downstream
regret as the primary metric.
The two portfolio benchmarks evaluate return-scenario generation for long-only
portfolio decisions. The synthetic sweep provides a controlled test under
increasing nonlinear context--return complexity, while the semi-real financial
task tests the method under chronological market data and realized downside-risk
evaluation. The PEMS-BAY task uses the same CVaR decision protocol for
congestion outcomes, providing an additional real-data test with different
outcome semantics and loss geometry. The detailed data generation process is described in Appendix \ref{app:exp detail}.

\paragraph{Baselines.} We compare DW-FM with four baselines. (i) {Uniform FM} uses the same conditional flow-matching architecture,
sampling procedure, and downstream solver as DW-FM, but trains with the standard
unweighted FM regression loss; this is the direct ablation for decision-sensitive
reweighting. (ii) {2Stage PTO} first trains a deterministic predictor using a
supervised prediction loss and then passes the prediction to the same downstream
optimizer~\citep{elmachtoub2022smart}; this represents the classical predict-then-optimize pipeline.
(iii) {SPO+} trains the deterministic predictor with the standard convex
decision-focused surrogate for predict-then-optimize learning~\citep{elmachtoub2022smart}. (iv) {Task-based
E2E} trains the predictor directly through the downstream task loss using the
same decision objective~\citep{donti2017task}. Within each benchmark, all methods are evaluated with the same frozen
task-specific CVaR regret evaluator.

\begin{table*}[htp]
\centering
\small
\caption{{Synthetic controlled portfolio sweep.}
Full-test and hardest-25\% regret across polynomial degrees. Lower is better.
Values are mean $\pm$ standard deviation over repeated runs.}
\label{tab:degree_sweep_combined}
\resizebox{\textwidth}{!}{
\begin{tabular}{llccccc}
\toprule
Split / Task & Degree & Uniform FM & 2Stage PTO & SPO+ & Task-based E2E & DW-FM \\
\midrule
\multirow{4}{*}{Full test (synthetic)}
& Deg-2 & $0.0745 \pm 0.0008$ & $0.0774 \pm 0.0008$ & $0.0912 \pm 0.0007$ & $0.0871 \pm 0.0022$ & $\mathbf{0.0726 \pm 0.0016}$ \\
& Deg-4 & $0.0773 \pm 0.0024$ & $0.0791 \pm 0.0016$ & $0.0913 \pm 0.0037$ & $0.0880 \pm 0.0046$ & $\mathbf{0.0739 \pm 0.0007}$ \\
& Deg-6 & $0.0779 \pm 0.0006$ & $0.0820 \pm 0.0034$ & $0.0925 \pm 0.0020$ & $0.0890 \pm 0.0010$ & $\mathbf{0.0756 \pm 0.0028}$ \\
& Deg-8 & $0.0766 \pm 0.0012$ & $0.0800 \pm 0.0021$ & $0.0918 \pm 0.0033$ & $0.0920 \pm 0.0038$ & $\mathbf{0.0746 \pm 0.0010}$ \\
\midrule
\multirow{4}{*}{Hardest 25\% (synthetic)}
& Deg-2 & $0.0649 \pm 0.0009$ & $0.0676 \pm 0.0008$ & $0.0768 \pm 0.0016$ & $0.0755 \pm 0.0015$ & $\mathbf{0.0633 \pm 0.0003}$ \\
& Deg-4 & $0.0662 \pm 0.0012$ & $0.0699 \pm 0.0012$ & $0.0789 \pm 0.0015$ & $0.0759 \pm 0.0040$ & $\mathbf{0.0650 \pm 0.0011}$ \\
& Deg-6 & $0.0647 \pm 0.0005$ & $0.0692 \pm 0.0018$ & $0.0764 \pm 0.0030$ & $0.0753 \pm 0.0023$ & $\mathbf{0.0638 \pm 0.0012}$ \\
& Deg-8 & $0.0657 \pm 0.0008$ & $0.0687 \pm 0.0012$ & $0.0775 \pm 0.0019$ & $0.0787 \pm 0.0014$ & $\mathbf{0.0641 \pm 0.0003}$ \\
\bottomrule
\end{tabular}}
\end{table*}

\begin{figure*}[htp]
    \centering
    \includegraphics[width=\textwidth]{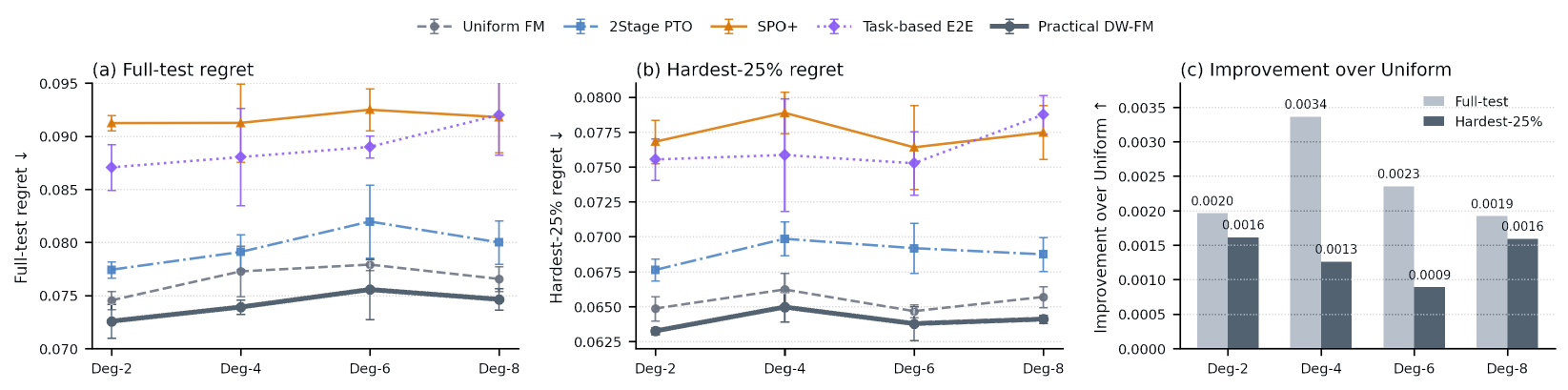}
    \caption{Synthetic controlled nonlinear portfolio degree sweep.
    {(a)} Full-test regret. 
    {(b)} Hardest-25\% regret. 
    {(c)} Improvement over Uniform FM. 
    Lower is better in (a,b); positive is better in (c).}
    \label{fig:degree_sweep}
\end{figure*}

\subsection{Experimental Results}

\paragraph{Synthetic controlled portfolio sweep benchmark.}
We first evaluate DW-FM on a fully synthetic portfolio-CVaR benchmark, where the
context-to-return map has polynomial degree in $\{2,4,6,8\}$. This benchmark
isolates decision-weighted training in a controlled setting: the downstream
portfolio-CVaR objective is fixed, while the conditional return model becomes
increasingly nonlinear. 
Table~\ref{tab:degree_sweep_combined} reports exact regret values with standard
deviations, while Figure~\ref{fig:degree_sweep} visualizes the comparison across
degrees and highlights improvement over Uniform FM. The hardest-25\% subset is defined by a
decision-sensitivity score that measures the estimated impact
on the downstream decision. Across all degrees and both
full-test and hardest-25\% splits, DW-FM attains the lowest regret among Uniform
FM, 2Stage PTO, SPO+, and Task-based E2E. The comparison with Uniform FM
isolates the effect of decision-sensitive reweighting, since both methods use
the same conditional flow-matching backbone and downstream evaluator. The
hardest-25\% results show that the improvement is preserved on
decision-sensitive contexts rather than being driven only by easy cases.

\paragraph{Semi-real portfolio-CVaR benchmark.} We next evaluate DW-FM on the semi-real financial portfolio-CVaR benchmark.
This task uses chronological market data and tests
whether decision-sensitive reweighting remains useful under realistic temporal
structure.
Table~\ref{tab:semireal_summary} reports two complementary views of the
semi-real financial experiment. The top block gives the primary full-test
decision-performance comparison under the frozen portfolio-CVaR evaluator.
DW-FM reduces mean regret from $0.00655$ to $0.00590$ and CVaR loss from
$0.01485$ to $0.01452$, while realized mean return remains essentially
unchanged. This suggests that the improvement comes mainly from better downside
risk control rather than from a return-seeking artifact.
The bottom block serves a different purpose: it is a tail-fit diagnostic on the
hardest 25\% contexts. We use
this subset to examine whether DW-FM improves the parts of the generated loss
distribution most relevant to the CVaR objective. DW-FM improves worst-10\% tail
Wasserstein, $q_{90}$ gap, and CVaR$_{90}$ gap, supporting the interpretation
that its regret gain is associated with better modeling of downside-tail regions.


\begin{table}[htp]
\centering
\small
\caption{{Semi-real portfolio-CVaR decision performance and tail diagnostics.}
The top block reports full-test decision metrics. The bottom block reports downside-tail fit diagnostics on the
predefined hardest 25\% contexts, used to assess whether DW-FM better captures
decision-relevant tail regions. Lower is better except mean return.}
\label{tab:semireal_summary}
\begin{tabular}{llccc}
\toprule
Group & Metric & Uniform FM & DW-FM & Improvement \\
\midrule
\multirow{4}{*}{Decision performance, full test}
& Mean regret & 0.00655 & \textbf{0.00590} & 0.00065 \\
& Std regret & 0.00048 & 0.00039 & -- \\
& Mean return & 0.00040 & 0.00039 & -- \\
& CVaR loss & 0.01485 & \textbf{0.01452} & 0.00033 \\
\midrule
\multirow{3}{*}{Tail diagnostics, hardest 25\%}
& Tail $W_1$ worst 10\% & 0.05195 & \textbf{0.04407} & 0.00788 \\
& $q_{90}$ gap & 0.03191 & \textbf{0.02561} & 0.00630 \\
& CVaR$_{90}$ gap & 0.05212 & \textbf{0.04377} & 0.00835 \\
\bottomrule
\end{tabular}

\vspace{1mm}
\footnotesize Context-level bridge on hardest 25\%: Spearman$(\Delta \text{tail-fit}, \Delta \text{regret})=0.3211$.
\end{table}

\begin{figure}[htp]
    \centering
    \includegraphics[width=\linewidth]{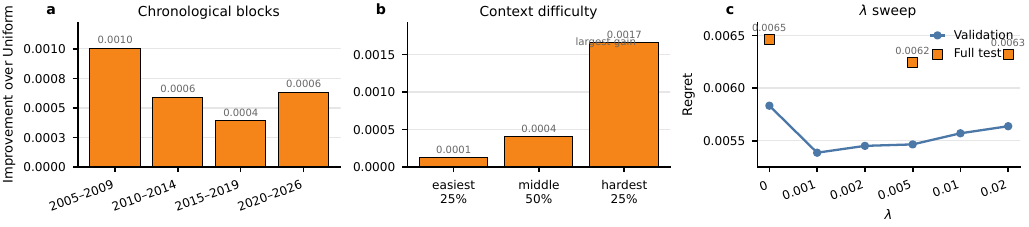}
    \caption{{Semi-real stability, localization, and robustness.} 
   {(a)} Chronological slices. 
    {(b)} Context difficulty. 
    {(c)} Lambda sweep.}
    \label{fig:robust_local}
\end{figure}

\paragraph{Semi-real portfolio-CVaR: stability, difficulty, and tuning.}
Figure~\ref{fig:robust_local} gives three diagnostics for the semi-real
portfolio-CVaR benchmark. Panel~(a) shows that DW-FM improves over Uniform FM in
all chronological test slices, suggesting that the full-test gain is not driven
by a single market period. Panel~(b) stratifies contexts by difficulty and shows
that the improvement is largest on the hardest contexts, where portfolio
decisions are most sensitive to distributional errors. Panel~(c) reports the
validation sweep over $\lambda$: $\lambda=0$ recovers Uniform FM, while positive
values improve regret up to a moderate range, consistent with the bias--fit
tradeoff induced by decision weighting.

\paragraph{Semi-real portfolio-CVaR: downside-tail mechanism.}
Table~\ref{tab:semireal_summary} shows that DW-FM improves downside-tail
diagnostics on the hardest contexts. Figure~\ref{fig:tail_bridge} further links
these diagnostics to decision quality: DW-FM reduces hardest-context regret,
improves multiple tail-fit measures, and shows a positive context-level
association between tail-fit improvement and regret improvement. These results
support the interpretation that the regret gain comes from better modeling of
the downside regions that drive the portfolio-CVaR optimizer.


\begin{figure*}[t]
    \centering
    \includegraphics[width=\textwidth]{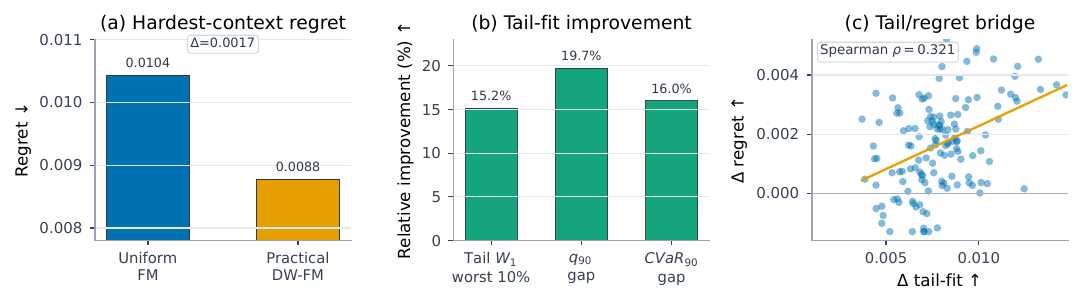}
   \caption{{Downside-risk mechanism on semi-real hardest contexts.}
{(a)} Hardest-context regret. 
{(b)} Relative improvement in downside-tail fit metrics.
{(c)} Context-level association between tail-fit improvement and regret improvement.
Positive $\Delta$ values indicate improvement over Uniform FM.}
    \label{fig:tail_bridge}
\end{figure*}
\paragraph{PEMS-BAY traffic-CVaR benchmark.}
We also evaluate DW-FM on the PEMS-BAY traffic-CVaR benchmark, which follows the
same predict-generate-optimize protocol but uses congestion outcomes rather than
asset returns. Table~\ref{tab:pemsbay} reports full-test regret under the frozen
traffic-CVaR evaluator. DW-FM obtains the lowest regret, reducing regret from
$758.71$ for Uniform FM to $729.73$, and also outperforming 2Stage PTO, SPO+,
and Task-based E2E. Because the outcome units and regret scale differ from the
portfolio benchmarks, we interpret this comparison within the PEMS-BAY task
rather than comparing absolute regret values across tasks.
\begin{table}[ht]
\centering
\small
\caption{{PEMS-BAY traffic-CVaR benchmark.}
Full-test regret under the frozen traffic-CVaR evaluator. Lower is better.}
\label{tab:pemsbay}
\begin{tabular}{lccccc}
\toprule
 & Uniform FM & 2Stage PTO & SPO+ & Task-based E2E & DW-FM \\
\midrule
Regret & 758.71 & 798.51 & 837.09 & 1754.45 & \textbf{729.73} \\
\bottomrule
\end{tabular}
\end{table}

\paragraph{Summary of experimental findings.}
Across the three benchmarks, DW-FM improves downstream regret relative
to the baselines. The synthetic sweep shows
robustness across nonlinear degrees, the semi-real portfolio task links regret
reduction to improved downside-tail fit, and PEMS-BAY provides an additional
real-data check under a different outcome geometry. Overall, the experiments
support the view that scenario generators for CSO should prioritize
decision-sensitive regions rather than uniform distributional fit.   

\section{Conclusion}
We identify and address the mismatch between uniform conditional generative
training and regret-driven contextual stochastic optimization. We propose
Decision-Weighted Flow Matching (DW-FM), which preserves standard flow matching
while reweighting training toward decision-sensitive regions. Our theory links
regret to decision discrepancy and pathwise velocity mismatch, motivating an
adjoint-weighted ideal objective and practical endpoint-weighted surrogates with
controlled tilting bias and regret bounds. Experiments in
synthetic and real data show improved downstream regret and downside-tail
behavior over standard baselines. DW-FM thus
offers a principled route to training conditional generators for decision quality
rather than uniform distributional fit.

\bibliographystyle{unsrtnat}
\bibliography{reference}

@book{shapiro2009lectures,
  title     = {Lectures on Stochastic Programming: Modeling and Theory},
  author    = {Shapiro, Alexander and Dentcheva, Darinka and Ruszczy{\'n}ski, Andrzej},
  series    = {MOS-SIAM Series on Optimization},
  volume    = {9},
  publisher = {Society for Industrial and Applied Mathematics and Mathematical Optimization Society},
  address   = {Philadelphia, PA},
  year      = {2009}
}

@article{rockafellar2000optimization,
  title   = {Optimization of Conditional Value-at-Risk},
  author  = {Rockafellar, R. Tyrrell and Uryasev, Stanislav},
  journal = {Journal of Risk},
  volume  = {2},
  number  = {3},
  pages   = {21--41},
  year    = {2000}
}

@inproceedings{donti2017task,
  title     = {Task-based End-to-end Model Learning in Stochastic Optimization},
  author    = {Donti, Priya L. and Amos, Brandon and Kolter, J. Zico},
  booktitle = {Advances in Neural Information Processing Systems},
  volume    = {30},
  year      = {2017}
}

@article{elmachtoub2022smart,
  title   = {Smart ``Predict, then Optimize''},
  author  = {Elmachtoub, Adam N. and Grigas, Paul},
  journal = {Management Science},
  volume  = {68},
  number  = {1},
  pages   = {9--26},
  year    = {2022},
  doi     = {10.1287/mnsc.2020.3922}
}

@inproceedings{wilder2019melding,
  title     = {Melding the Data-Decisions Pipeline: Decision-Focused Learning for Combinatorial Optimization},
  author    = {Wilder, Bryan and Dilkina, Bistra and Tambe, Milind},
  booktitle = {Proceedings of the AAAI Conference on Artificial Intelligence},
  volume    = {33},
  number    = {01},
  pages     = {1658--1665},
  year      = {2019},
  doi       = {10.1609/aaai.v33i01.33011658}
}

@inproceedings{ferber2020mipaal,
  title     = {MIPaaL: Mixed Integer Program as a Layer},
  author    = {Ferber, Aaron and Wilder, Bryan and Dilkina, Bistra and Tambe, Milind},
  booktitle = {Proceedings of the AAAI Conference on Artificial Intelligence},
  volume    = {34},
  number    = {02},
  pages     = {1504--1511},
  year      = {2020},
  doi       = {10.1609/aaai.v34i02.5509}
}

@inproceedings{lipman2023flow,
  title     = {Flow Matching for Generative Modeling},
  author    = {Lipman, Yaron and Chen, Ricky T. Q. and Ben-Hamu, Heli and Nickel, Maximilian and Le, Matthew},
  booktitle = {International Conference on Learning Representations},
  year      = {2023},
  url       = {https://openreview.net/forum?id=PqvMRDCJT9t}
}

@article{wang2025gendfl,
  title   = {Gen-DFL: Decision-Focused Generative Learning for Robust Decision Making},
  author  = {Wang, Prince Zizhuang and Chen, Shuyi and Liang, Jinhao and Fioretto, Ferdinando and Zhu, Shixiang},
  journal = {arXiv preprint arXiv:2502.05468},
  year    = {2025}
}

@article{zhao2025diffusiondfl,
  title   = {Diffusion-DFL: Decision-focused Diffusion Models for Stochastic Optimization},
  author  = {Zhao, Zihao and Yeh, Christopher and Kong, Lingkai and Wang, Kai},
  journal = {arXiv preprint arXiv:2510.11590},
  year    = {2025}
}

@inproceedings{agrawal2019differentiable,
  title     = {Differentiable Convex Optimization Layers},
  author    = {Agrawal, Akshay and Amos, Brandon and Barratt, Shane and Boyd, Stephen and Diamond, Steven and Kolter, J. Zico},
  booktitle = {Advances in Neural Information Processing Systems},
  volume    = {32},
  year      = {2019}
}

@inproceedings{tong2024improving,
  title     = {Improving and Generalizing Flow-Based Generative Models with Minibatch Optimal Transport},
  author    = {Tong, Alexander and Fatras, Kilian and Malkin, Nikolay and Huguet, Guillaume and Zhang, Yanlei and Rector-Brooks, Jarrid and Wolf, Guy and Bengio, Yoshua},
  booktitle = {Transactions on Machine Learning Research},
  year      = {2024},
  url       = {https://openreview.net/forum?id=CD9Snc73AW}
}

@inproceedings{elkan2001foundations,
  title     = {The Foundations of Cost-Sensitive Learning},
  author    = {Elkan, Charles},
  booktitle = {Proceedings of the Seventeenth International Joint Conference on Artificial Intelligence},
  pages     = {973--978},
  year      = {2001}
}

@inproceedings{zadrozny2003cost,
  title     = {Cost-Sensitive Learning by Cost-Proportionate Example Weighting},
  author    = {Zadrozny, Bianca and Langford, John and Abe, Naoki},
  booktitle = {Proceedings of the Third IEEE International Conference on Data Mining},
  pages     = {435--442},
  year      = {2003},
  doi       = {10.1109/ICDM.2003.1250950}
}

@book{santambrogio2015optimal,
  title={Optimal Transport for Applied Mathematicians: Calculus of Variations, PDEs, and Modeling},
  author={Santambrogio, Filippo},
  volume={87},
  year={2015},
  publisher={Birkh{\"a}user}
}

@article{cances2020variational,
  title={A variational finite volume scheme for Wasserstein gradient flows},
  author={Cances, Cl{\'e}ment and Gallou{\"e}t, Thomas O and Todeschi, Gabriele},
  journal={Numerische Mathematik},
  volume={146},
  number={3},
  pages={437--480},
  year={2020},
  publisher={Springer}
}

@article{fama1993common,
  title={Common risk factors in the returns on stocks and bonds},
  author={Fama, Eugene F and French, Kenneth R},
  journal={Journal of Financial Economics},
  volume={33},
  number={1},
  pages={3--56},
  year={1993},
  publisher={Elsevier}
}

@misc{kenfrench,
  author = {Kenneth R. French},
  title = {Ken French Data Library},
  howpublished = {\url{http://mba.tuck.dartmouth.edu/pages/faculty/ken.french/data_library.html}},
  year = {2025}
}

@misc{pemsbay,
  title = {PEMS-BAY Traffic Data},
  howpublished = {\url{https://pems.dot.ca.gov/}}, 
  year = {2025}
}

@article{bertsimas2020predictive,
  title={From predictive to prescriptive analytics},
  author={Bertsimas, Dimitris and Kallus, Nathan},
  journal={Management Science},
  volume={66},
  number={3},
  pages={1025--1044},
  year={2020},
  publisher={INFORMS}
}

@article{kallus2023stochastic,
  title={Stochastic optimization forests},
  author={Kallus, Nathan and Mao, Xiaojie},
  journal={Management Science},
  volume={69},
  number={4},
  pages={1975--1994},
  year={2023},
  publisher={INFORMS}
}

@inproceedings{amos2017optnet,
  title={Optnet: Differentiable optimization as a layer in neural networks},
  author={Amos, Brandon and Kolter, J Zico},
  booktitle={International conference on machine learning},
  pages={136--145},
  year={2017},
  organization={PMLR}
}

@article{mandi2024decision,
  title={Decision-focused learning: Foundations, state of the art, benchmark and future opportunities},
  author={Mandi, Jayanta and Kotary, James and Berden, Senne and Mulamba, Maxime and Bucarey, Victor and Guns, Tias and Fioretto, Ferdinando},
  journal={Journal of Artificial Intelligence Research},
  volume={80},
  pages={1623--1701},
  year={2024}
}

@article{berthet2020learning,
  title={Learning with differentiable pertubed optimizers},
  author={Berthet, Quentin and Blondel, Mathieu and Teboul, Olivier and Cuturi, Marco and Vert, Jean-Philippe and Bach, Francis},
  journal={Advances in neural information processing systems},
  volume={33},
  pages={9508--9519},
  year={2020}
}

@inproceedings{rasul2021autoregressive,
  title={Autoregressive denoising diffusion models for multivariate probabilistic time series forecasting},
  author={Rasul, Kashif and Seward, Calvin and Schuster, Ingmar and Vollgraf, Roland},
  booktitle={International conference on machine learning},
  pages={8857--8868},
  year={2021},
  organization={PMLR}
}

@article{yan2021scoregrad,
  title={Scoregrad: Multivariate probabilistic time series forecasting with continuous energy-based generative models},
  author={Yan, Tijin and Zhang, Hongwei and Zhou, Tong and Zhan, Yufeng and Xia, Yuanqing},
  journal={arXiv preprint arXiv:2106.10121},
  year={2021}
}

@article{liu2022flow,
  title={Flow straight and fast: Learning to generate and transfer data with rectified flow},
  author={Liu, Xingchao and Gong, Chengyue and Liu, Qiang},
  journal={arXiv preprint arXiv:2209.03003},
  year={2022}
}

@article{albergo2022building,
  title={Building normalizing flows with stochastic interpolants},
  author={Albergo, Michael S and Vanden-Eijnden, Eric},
  journal={arXiv preprint arXiv:2209.15571},
  year={2022}
}

@article{shimodaira2000improving,
  title={Improving predictive inference under covariate shift by weighting the log-likelihood function},
  author={Shimodaira, Hidetoshi},
  journal={Journal of statistical planning and inference},
  volume={90},
  number={2},
  pages={227--244},
  year={2000},
  publisher={Elsevier}
}

@article{sugiyama2007covariate,
  title={Covariate shift adaptation by importance weighted cross validation.},
  author={Sugiyama, Masashi and Krauledat, Matthias and M{\"u}ller, Klaus-Robert},
  journal={Journal of Machine Learning Research},
  volume={8},
  number={5},
  year={2007}
}

@inproceedings{byrd2019effect,
  title={What is the effect of importance weighting in deep learning?},
  author={Byrd, Jonathon and Lipton, Zachary},
  booktitle={International conference on machine learning},
  pages={872--881},
  year={2019},
  organization={PMLR}
}

@article{albergo2025stochastic,
  title={Stochastic interpolants: A unifying framework for flows and diffusions},
  author={Albergo, Michael and Boffi, Nicholas M and Vanden-Eijnden, Eric},
  journal={Journal of Machine Learning Research},
  volume={26},
  number={209},
  pages={1--80},
  year={2025}
}

\newpage
\appendix

\section{Technical Proofs}
\label{app:A}
\begin{algorithm}[ht]
\caption{Decision-Weighted Flow Matching (DW-FM)}
\label{alg:DW-FM}
\begin{algorithmic}[1]
\Require Training data $\mathcal{D}=\{(\mathbf{x}_i,\mathbf{s}_i)\}_{i=1}^n$; base sampler $q_{0,\mathbf{x}}$; reference routine $\mathsf{Ref}$; weight parameter $\lambda\ge0$; vector field $\mathbf{v}_\theta$.
\Ensure Learned vector field $\mathbf{v}_\theta$.
\While{not converged}
    \State Sample a minibatch $\mathcal{B}\subseteq\mathcal{D}$ of pairs $(\mathbf{x},\mathbf{s}_1)$.
    \State For each $(\mathbf{x},\mathbf{s}_1)\in\mathcal{B}$, set
    $\widehat{\mathbf{z}}_{\mathbf{x}}\gets\mathsf{Ref}(\mathbf{x})$ and
    $\widehat w_{\mathbf{x}}(\mathbf{s}_1)
    \gets
    1+\lambda\|\nabla_{\mathbf{s}}\ell_{\mathbf{x}}(\widehat{\mathbf{z}}_{\mathbf{x}},\mathbf{s}_1)\|^2$.
    \State Sample $\mathbf{s}_0\sim q_{0,\mathbf{x}}$ and $t\sim\mathrm{Unif}[0,1]$ for each pair; set
    $\mathbf{s}_t=(1-t)\mathbf{s}_0+t\mathbf{s}_1$ and
    $\boldsymbol{\Delta}=\mathbf{s}_1-\mathbf{s}_0$.
    \State Update $\theta$ by a stochastic gradient step on
    \[
    \frac{1}{|\mathcal{B}|}
    \sum_{(\mathbf{x},\mathbf{s}_1)\in\mathcal{B}}
    \widehat w_{\mathbf{x}}(\mathbf{s}_1)
    \left\|
    \mathbf{v}_\theta(t,\mathbf{s}_t,\mathbf{x})-\boldsymbol{\Delta}
    \right\|^2 .
    \]
\EndWhile
\State \Return $\mathbf{v}_\theta$.
\end{algorithmic}
\end{algorithm}

\subsection{Proof of Proposition~\ref{prop:regret-discrepancy}}

Fix a context $\mathbf{x}$. Recall that
\[
\mathbf{z}_{\theta}(\mathbf{x})
\in
\argmin_{\mathbf{z}\in\mathcal{Z}}
R_{\mathbf{x}}(\mathbf{z};q_{\theta,\mathbf{x}})
\]
is the plug-in decision under the learned law, while
\[
\mathbf{z}_{\mathbf{x}}^\star
\in
\argmin_{\mathbf{z}\in\mathcal{Z}}
R_{\mathbf{x}}(\mathbf{z};q_{\mathbf{x}}^\star)
\]
is the population-optimal decision under the true law. Since
$\mathbf{z}_{\theta}(\mathbf{x})$ minimizes
$R_{\mathbf{x}}(\cdot;q_{\theta,\mathbf{x}})$, we have
\[
R_{\mathbf{x}}\!\left(\mathbf{z}_{\theta}(\mathbf{x});q_{\theta,\mathbf{x}}\right)
\le
R_{\mathbf{x}}\!\left(\mathbf{z}_{\mathbf{x}}^\star;q_{\theta,\mathbf{x}}\right).
\]
Therefore,
\begin{align*}
\mathrm{Reg}_{\mathbf{x}}(\theta)
&=
R_{\mathbf{x}}\!\left(\mathbf{z}_{\theta}(\mathbf{x});q_{\mathbf{x}}^\star\right)
-
R_{\mathbf{x}}\!\left(\mathbf{z}_{\mathbf{x}}^\star;q_{\mathbf{x}}^\star\right)\\
&=
\Big[
R_{\mathbf{x}}\!\left(\mathbf{z}_{\theta}(\mathbf{x});q_{\mathbf{x}}^\star\right)
-
R_{\mathbf{x}}\!\left(\mathbf{z}_{\theta}(\mathbf{x});q_{\theta,\mathbf{x}}\right)
\Big]+
\Big[
R_{\mathbf{x}}\!\left(\mathbf{z}_{\theta}(\mathbf{x});q_{\theta,\mathbf{x}}\right)
-
R_{\mathbf{x}}\!\left(\mathbf{z}_{\mathbf{x}}^\star;q_{\theta,\mathbf{x}}\right)
\Big]\\
&\quad+
\Big[
R_{\mathbf{x}}\!\left(\mathbf{z}_{\mathbf{x}}^\star;q_{\theta,\mathbf{x}}\right)
-
R_{\mathbf{x}}\!\left(\mathbf{z}_{\mathbf{x}}^\star;q_{\mathbf{x}}^\star\right)
\Big].
\end{align*}
The middle term is nonpositive by the optimality of
$\mathbf{z}_{\theta}(\mathbf{x})$ under $q_{\theta,\mathbf{x}}$. Hence
\begin{align*}
\mathrm{Reg}_{\mathbf{x}}(\theta)
&\le
\left|
R_{\mathbf{x}}\!\left(\mathbf{z}_{\theta}(\mathbf{x});q_{\mathbf{x}}^\star\right)
-
R_{\mathbf{x}}\!\left(\mathbf{z}_{\theta}(\mathbf{x});q_{\theta,\mathbf{x}}\right)
\right|+
\left|
R_{\mathbf{x}}\!\left(\mathbf{z}_{\mathbf{x}}^\star;q_{\theta,\mathbf{x}}\right)
-
R_{\mathbf{x}}\!\left(\mathbf{z}_{\mathbf{x}}^\star;q_{\mathbf{x}}^\star\right)
\right|.
\end{align*}
By the definition of $d_{\mathrm{dec},\mathbf{x}}$,it follows that
\[
\mathrm{Reg}_{\mathbf{x}}(\theta)
\le
2\,d_{\mathrm{dec},\mathbf{x}}(q_{\mathbf{x}}^\star,q_{\theta,\mathbf{x}}).
\]
Since this bound holds for every context $\mathbf{x}$, evaluating it at the random context
$\mathbf{X}$ and taking expectations gives
\[
\mathbb E\!\left[\mathrm{Reg}_{\mathbf{X}}(\theta)\right]
\le
2\,\mathbb E\!\left[
d_{\mathrm{dec},\mathbf{X}}(q_{\mathbf{X}}^\star,q_{\theta,\mathbf{X}})
\right].
\]
\hfill\(\square\)


\subsection{Proof of Proposition~\ref{prop:adjoint-identity}}

\begin{proposition}
\label{prop:adjoint-identity}
Under Assumptions~\ref{ass:path-adjoint-regularity}
and~\ref{ass:no-boundary-flux}, for every fixed decision
\(\mathbf z\in\mathcal Z\),
\[
\mathbb{E}_{q_{\theta,\mathbf{x}}}
\!\left[
\ell_{\mathbf{x}}(\mathbf z,\mathbf S)
\right]
-
\mathbb{E}_{q_{\mathbf{x}}^\star}
\!\left[
\ell_{\mathbf{x}}(\mathbf z,\mathbf S)
\right]
=
\int_0^1
\!\int_{\mathcal S}
\left\langle
\nabla_{\mathbf s}\phi_{\mathbf{x},\mathbf z}(t,\mathbf s),
\mathbf v_\theta(t,\mathbf s,\mathbf x)
-
\mathbf u_{\mathbf x}(t,\mathbf s)
\right\rangle
\,dq_{\mathbf{x},t}^{\theta}(\mathbf s)\,dt .
\]
\end{proposition}

\emph{Proof.}
Fix a context \(\mathbf x\) and a decision \(\mathbf z\in\mathcal Z\). Let
\(\{\mu_{\mathbf{x},t}\}_{t\in[0,1]}\) denote the target path in
Assumption~\ref{ass:path-adjoint-regularity}, so that
\(\mu_{\mathbf{x},0}=q_{0,\mathbf{x}}\) and
\(\mu_{\mathbf{x},1}=q_{\mathbf{x}}^\star\). Since the learned path
\(q_{\mathbf{x},t}^{\theta}\) is absolutely continuous, write
\(q_{\mathbf{x},t}^{\theta}(d\mathbf s)
=
\rho_{\mathbf{x},t}^{\theta}(\mathbf s)d\mathbf s\). Define
\[
I_{\mathbf z}(t)
:=
\int_{\mathcal S}
\phi_{\mathbf{x},\mathbf z}(t,\mathbf s)
\,dq_{\mathbf{x},t}^{\theta}(\mathbf s)
=
\int_{\mathcal S}
\phi_{\mathbf{x},\mathbf z}(t,\mathbf s)
\rho_{\mathbf{x},t}^{\theta}(\mathbf s)\,d\mathbf s .
\]
Differentiating in time and using the continuity equation
\(\partial_t\rho_{\mathbf{x},t}^{\theta}
+
\nabla_{\mathbf s}\cdot
(\rho_{\mathbf{x},t}^{\theta}\mathbf v_\theta)=0\), we obtain
\[
\frac{d}{dt}I_{\mathbf z}(t)
=
\int_{\mathcal S}
\partial_t\phi_{\mathbf{x},\mathbf z}(t,\mathbf s)
\rho_{\mathbf{x},t}^{\theta}(\mathbf s)\,d\mathbf s
-
\int_{\mathcal S}
\phi_{\mathbf{x},\mathbf z}(t,\mathbf s)
\nabla_{\mathbf s}\!\cdot
\left(
\rho_{\mathbf{x},t}^{\theta}(\mathbf s)
\mathbf v_\theta(t,\mathbf s,\mathbf x)
\right)
d\mathbf s .
\]
By integration by parts and the no-boundary-flux condition for the learned
path,
\[
-
\int_{\mathcal S}
\phi_{\mathbf{x},\mathbf z}(t,\mathbf s)
\nabla_{\mathbf s}\!\cdot
\left(
\rho_{\mathbf{x},t}^{\theta}(\mathbf s)
\mathbf v_\theta(t,\mathbf s,\mathbf x)
\right)
d\mathbf s
=
\int_{\mathcal S}
\left\langle
\nabla_{\mathbf s}\phi_{\mathbf{x},\mathbf z}(t,\mathbf s),
\mathbf v_\theta(t,\mathbf s,\mathbf x)
\right\rangle
\rho_{\mathbf{x},t}^{\theta}(\mathbf s)\,d\mathbf s .
\]
Moreover, the backward transport equation~\eqref{eq:adjoint-def-method} gives
\[
\partial_t\phi_{\mathbf{x},\mathbf z}(t,\mathbf s)
=
-
\left\langle
\mathbf u_{\mathbf x}(t,\mathbf s),
\nabla_{\mathbf s}\phi_{\mathbf{x},\mathbf z}(t,\mathbf s)
\right\rangle .
\]
Combining the last two displays,
\[
\frac{d}{dt}I_{\mathbf z}(t)
=
\int_{\mathcal S}
\left\langle
\nabla_{\mathbf s}\phi_{\mathbf{x},\mathbf z}(t,\mathbf s),
\mathbf v_\theta(t,\mathbf s,\mathbf x)
-
\mathbf u_{\mathbf x}(t,\mathbf s)
\right\rangle
\,dq_{\mathbf{x},t}^{\theta}(\mathbf s).
\]
Integrating over \(t\in[0,1]\) yields
\[
I_{\mathbf z}(1)-I_{\mathbf z}(0)
=
\int_0^1
\!\int_{\mathcal S}
\left\langle
\nabla_{\mathbf s}\phi_{\mathbf{x},\mathbf z}(t,\mathbf s),
\mathbf v_\theta(t,\mathbf s,\mathbf x)
-
\mathbf u_{\mathbf x}(t,\mathbf s)
\right\rangle
\,dq_{\mathbf{x},t}^{\theta}(\mathbf s)\,dt .
\]

At terminal time, \(\phi_{\mathbf{x},\mathbf z}(1,\mathbf s)
=
\ell_{\mathbf{x}}(\mathbf z,\mathbf s)\) and
\(q_{\mathbf{x},1}^{\theta}=q_{\theta,\mathbf{x}}\). Hence
\[
I_{\mathbf z}(1)
=
\mathbb E_{q_{\theta,\mathbf{x}}}
\!\left[
\ell_{\mathbf{x}}(\mathbf z,\mathbf S)
\right].
\]

It remains to identify \(I_{\mathbf z}(0)\). Define the analogous transported
quantity along the target path:
\[
J_{\mathbf z}(t)
:=
\int_{\mathcal S}
\phi_{\mathbf{x},\mathbf z}(t,\mathbf s)
\,d\mu_{\mathbf{x},t}(\mathbf s).
\]
Using the continuity equation for \(\mu_{\mathbf{x},t}\), the same integration
by parts argument, and the no-boundary-flux condition for the target path gives
\[
\frac{d}{dt}J_{\mathbf z}(t)=0.
\]
Therefore \(J_{\mathbf z}(0)=J_{\mathbf z}(1)\). Since
\(\mu_{\mathbf{x},1}=q_{\mathbf{x}}^\star\),
\[
J_{\mathbf z}(1)
=
\mathbb E_{q_{\mathbf{x}}^\star}
\!\left[
\ell_{\mathbf{x}}(\mathbf z,\mathbf S)
\right].
\]
Since the learned and target paths share the same initial distribution,
\(q_{\mathbf{x},0}^{\theta}=\mu_{\mathbf{x},0}=q_{0,\mathbf{x}}\), we also have
\[
I_{\mathbf z}(0)
=
\int_{\mathcal S}
\phi_{\mathbf{x},\mathbf z}(0,\mathbf s)
\,dq_{\mathbf{x},0}^{\theta}(\mathbf s)
=
\int_{\mathcal S}
\phi_{\mathbf{x},\mathbf z}(0,\mathbf s)
\,d\mu_{\mathbf{x},0}(\mathbf s)
=
J_{\mathbf z}(0).
\]
Combining the last two displays gives
\[
I_{\mathbf z}(0)
=
\mathbb E_{q_{\mathbf{x}}^\star}
\!\left[
\ell_{\mathbf{x}}(\mathbf z,\mathbf S)
\right].
\]
Substituting the endpoint identities for \(I_{\mathbf z}(1)\) and
\(I_{\mathbf z}(0)\) into the integrated identity proves the proposition.
\hfill\(\square\)
\subsection{Proof of Theorem~\ref{thm:ideal-controls-regret}}

Fix a context $\mathbf{x}$ and a decision $\mathbf{z} \in \mathcal Z$. By Proposition~\ref{prop:adjoint-identity},
\[
\mathbb E_{q_{\theta,\mathbf{x}}}\!\left[\ell_{\mathbf{x}}(\mathbf{z},\mathbf{S})\right]
-
\mathbb E_{q_{\mathbf{x}}^\star}\!\left[\ell_{\mathbf{x}}(\mathbf{z},\mathbf{S})\right]
=
\int_0^1
\int_{\mathcal S}
\left\langle
\nabla_{\mathbf{s}} \phi_{\mathbf{x},\mathbf{z}}(t,\mathbf{s}),
\, \mathbf{v}_\theta(t,\mathbf{s},\mathbf{x})-\mathbf{u}_{\mathbf{x}}(t,\mathbf{s})
\right\rangle
\, dq_{\mathbf{x},t}^\theta(\mathbf{s})\, dt.
\]
Taking absolute values and applying the Cauchy--Schwarz inequality with respect to the probability measure
$dq_{\mathbf{x},t}^\theta(\mathbf{s})\,dt$, we obtain
\begin{align*}
&\left|
\mathbb E_{q_{\theta,\mathbf{x}}}\!\left[\ell_{\mathbf{x}}(\mathbf{z},\mathbf{S})\right]
-
\mathbb E_{q_{\mathbf{x}}^\star}\!\left[\ell_{\mathbf{x}}(\mathbf{z},\mathbf{S})\right]
\right| \\
&\le
\int_0^1
\int_{\mathcal S}
\left\|
\nabla_{\mathbf{s}} \phi_{\mathbf{x},\mathbf{z}}(t,\mathbf{s})
\right\|
\left\|
\mathbf{v}_\theta(t,\mathbf{s},\mathbf{x})-\mathbf{u}_{\mathbf{x}}(t,\mathbf{s})
\right\|
\, dq_{\mathbf{x},t}^\theta(\mathbf{s})\, dt \\
&\le
\left(
\int_0^1
\int_{\mathcal S}
\left\|
\nabla_{\mathbf{s}} \phi_{\mathbf{x},\mathbf{z}}(t,\mathbf{s})
\right\|^2
\left\|
\mathbf{v}_\theta(t,\mathbf{s},\mathbf{x})-\mathbf{u}_{\mathbf{x}}(t,\mathbf{s})
\right\|^2
\, dq_{\mathbf{x},t}^\theta(\mathbf{s})\, dt
\right)^{1/2}
\left(
\int_0^1
\int_{\mathcal S}
dq_{\mathbf{x},t}^\theta(\mathbf{s})\, dt
\right)^{1/2}.
\end{align*}
Since $q_{\mathbf{x},t}^\theta$ is a probability law on $\mathcal S$ for every $t \in [0,1]$,
\[
\int_0^1
\int_{\mathcal S}
dq_{\mathbf{x},t}^\theta(\mathbf{s})\, dt
=
\int_0^1 1\,dt
=
1.
\]
Moreover, by the definition of the envelope
\[
M_{\mathbf{x}}(t,\mathbf{s})
:=
\sup_{\mathbf{z}\in\mathcal Z}
\left\|
\nabla_{\mathbf{s}} \phi_{\mathbf{x},\mathbf{z}}(t,\mathbf{s})
\right\|^2,
\]
we have
\[
\left\|
\nabla_{\mathbf{s}} \phi_{\mathbf{x},\mathbf{z}}(t,\mathbf{s})
\right\|^2
\le
M_{\mathbf{x}}(t,\mathbf{s}).
\]
Therefore,
\begin{align*}
\left|
\mathbb E_{q_{\theta,\mathbf{x}}}\!\left[\ell_{\mathbf{x}}(\mathbf{z},\mathbf{S})\right]
-
\mathbb E_{q_{\mathbf{x}}^\star}\!\left[\ell_{\mathbf{x}}(\mathbf{z},\mathbf{S})\right]
\right|
&\le
\left(
\int_0^1
\int_{\mathcal S}
M_{\mathbf{x}}(t,\mathbf{s})
\left\|
\mathbf{v}_\theta(t,\mathbf{s},\mathbf{x})-\mathbf{u}_{\mathbf{x}}(t,\mathbf{s})
\right\|^2
\, dq_{\mathbf{x},t}^\theta(\mathbf{s})\, dt
\right)^{1/2} \\
&=
L_{\mathrm{ideal},\mathbf{x}}(\theta)^{1/2}.
\end{align*}
Since this bound holds for every $\mathbf{z}\in\mathcal Z$, taking the supremum over $\mathbf{z}$ yields
\[
d_{\mathrm{dec},\mathbf{x}}\!\left(q_{\mathbf{x}}^\star,q_{\theta,\mathbf{x}}\right)
\le
L_{\mathrm{ideal},\mathbf{x}}(\theta)^{1/2}.
\]
The regret bound then follows from Proposition~\ref{prop:regret-discrepancy}:
\[
\mathrm{Reg}_{\mathbf{x}}(\theta)
\le
2\,d_{\mathrm{dec},\mathbf{x}}\!\left(q_{\mathbf{x}}^\star,q_{\theta,\mathbf{x}}\right)
\le
2\,L_{\mathrm{ideal},\mathbf{x}}(\theta)^{1/2}.
\]
Finally, evaluating at the random context $\mathbf{X}$ and taking expectations gives
\[
\mathbb E\!\left[\mathrm{Reg}_{\mathbf{X}}(\theta)\right]
\le
2\,\mathbb E\!\left[L_{\mathrm{ideal},\mathbf{X}}(\theta)^{1/2}\right].
\]
This completes the proof. \hfill\(\square\)

\subsection{Proof of Theorem~\ref{thm:endpoint-population-target}}

First fix a regression input value $\mathbf{Y}=\mathbf{y}$ such that
$m_{w,\mathbf{x}}(\mathbf{y})>0$. To simplify notation inside the proof, write
\[
m(\mathbf{y})
:=
m_{w,\mathbf{x}}(\mathbf{y})
=
\mathbb{E}_{\mathbf{x}}
\!\left[
w_{\mathbf{x}}(\mathbf{S}_1)\mid \mathbf{Y}=\mathbf{y}
\right],
\]
and
\[
\mathbf{u}_w(\mathbf{y})
:=
\mathbf{u}_{w,\mathbf{x}}(\mathbf{y})
=
\frac{
\mathbb{E}_{\mathbf{x}}
\!\left[
w_{\mathbf{x}}(\mathbf{S}_1)\boldsymbol{\Delta}
\mid
\mathbf{Y}=\mathbf{y}
\right]
}{
m(\mathbf{y})
}.
\]
For any measurable vector field $\mathbf{v}$, define
\[
\mathbf{a}
:=
\mathbf{v}(\mathbf{y})-\mathbf{u}_w(\mathbf{y}),
\qquad
\mathbf{b}
:=
\mathbf{u}_w(\mathbf{y})-\boldsymbol{\Delta}.
\]
Then
\[
\left\|
\mathbf{v}(\mathbf{y})-\boldsymbol{\Delta}
\right\|^2
=
\left\|
\mathbf{a}+\mathbf{b}
\right\|^2
=
\|\mathbf{a}\|^2
+
2\langle \mathbf{a},\mathbf{b}\rangle
+
\|\mathbf{b}\|^2.
\]
Multiplying by $w_{\mathbf{x}}(\mathbf{S}_1)$ and conditioning on
$\mathbf{Y}=\mathbf{y}$ gives
\begin{align*}
&
\mathbb{E}_{\mathbf{x}}
\!\left[
w_{\mathbf{x}}(\mathbf{S}_1)
\left\|
\mathbf{v}(\mathbf{y})-\boldsymbol{\Delta}
\right\|^2
\mid
\mathbf{Y}=\mathbf{y}
\right] \\
&=
m(\mathbf{y})
\left\|
\mathbf{v}(\mathbf{y})-\mathbf{u}_w(\mathbf{y})
\right\|^2
+
2
\left\langle
\mathbf{v}(\mathbf{y})-\mathbf{u}_w(\mathbf{y}),
\mathbb{E}_{\mathbf{x}}
\!\left[
w_{\mathbf{x}}(\mathbf{S}_1)\mathbf{b}
\mid
\mathbf{Y}=\mathbf{y}
\right]
\right\rangle \\
&\quad+
\mathbb{E}_{\mathbf{x}}
\!\left[
w_{\mathbf{x}}(\mathbf{S}_1)
\left\|
\mathbf{u}_w(\mathbf{y})-\boldsymbol{\Delta}
\right\|^2
\mid
\mathbf{Y}=\mathbf{y}
\right].
\end{align*}
Note that
\begin{align*}
\mathbb{E}_{\mathbf{x}}
\!\left[
w_{\mathbf{x}}(\mathbf{S}_1)\mathbf{b}
\mid
\mathbf{Y}=\mathbf{y}
\right]
&=
\mathbf{u}_w(\mathbf{y})
\mathbb{E}_{\mathbf{x}}
\!\left[
w_{\mathbf{x}}(\mathbf{S}_1)
\mid
\mathbf{Y}=\mathbf{y}
\right]
-
\mathbb{E}_{\mathbf{x}}
\!\left[
w_{\mathbf{x}}(\mathbf{S}_1)\boldsymbol{\Delta}
\mid
\mathbf{Y}=\mathbf{y}
\right] \\
&=
m(\mathbf{y})\mathbf{u}_w(\mathbf{y})
-
m(\mathbf{y})\mathbf{u}_w(\mathbf{y})
=
\mathbf{0}.
\end{align*}
Therefore,
\[
\mathbb{E}_{\mathbf{x}}
\!\left[
w_{\mathbf{x}}(\mathbf{S}_1)
\left\|
\mathbf{v}(\mathbf{y})-\boldsymbol{\Delta}
\right\|^2
\mid
\mathbf{Y}=\mathbf{y}
\right]
=
m(\mathbf{y})
\left\|
\mathbf{v}(\mathbf{y})-\mathbf{u}_w(\mathbf{y})
\right\|^2
+
C(\mathbf{y}),
\]
where
\[
C(\mathbf{y})
:=
\mathbb{E}_{\mathbf{x}}
\!\left[
w_{\mathbf{x}}(\mathbf{S}_1)
\left\|
\mathbf{u}_w(\mathbf{y})-\boldsymbol{\Delta}
\right\|^2
\mid
\mathbf{Y}=\mathbf{y}
\right]
\]
does not depend on $\mathbf{v}(\mathbf{y})$. Since
$m(\mathbf{y})>0$, the conditional risk is minimized uniquely at
$\mathbf{v}(\mathbf{y})=\mathbf{u}_w(\mathbf{y})$.

On the set where $m_{w,\mathbf{x}}(\mathbf{Y})=0$, the conditional weighted loss is zero for every
choice of $\mathbf{v}(\mathbf{Y})$, because
$w_{\mathbf{x}}(\mathbf{S}_1)\ge 0$ and
$\mathbb{E}_{\mathbf{x}}[w_{\mathbf{x}}(\mathbf{S}_1)\mid \mathbf{Y}]=0$ imply
$w_{\mathbf{x}}(\mathbf{S}_1)=0$ conditionally almost surely. Hence the value of
$\mathbf{u}_{w,\mathbf{x}}$ on this set is irrelevant.

Averaging the conditional decomposition over $\mathbf{Y}$ also gives the stronger identity
\begin{equation}
L_{w,\mathbf{x}}(\mathbf{v})
-
L_{w,\mathbf{x}}(\mathbf{u}_{w,\mathbf{x}})
=
\mathbb{E}_{\mathbf{x}}
\!\left[
m_{w,\mathbf{x}}(\mathbf{Y})
\left\|
\mathbf{v}(\mathbf{Y})-\mathbf{u}_{w,\mathbf{x}}(\mathbf{Y})
\right\|^2
\right].\label{eq: weighted loss identity}
\end{equation}
Therefore any population minimizer of \(L_{w,\mathbf{x}}\) agrees with
\(\mathbf{u}_{w,\mathbf{x}}\) almost surely under the conditional distribution
of \(\mathbf{Y}\) given \(\mathbf{X}=\mathbf{x}\) on
\(\{m_{w,\mathbf{x}}(\mathbf{Y})>0\}\), and any vector field with this property
is a population minimizer.

It remains to prove the covariance identity. By definition,
\[
\mathbf{u}_{\mathbf{x}}(\mathbf{Y})
=
\mathbb{E}_{\mathbf{x}}
\!\left[
\boldsymbol{\Delta}
\mid
\mathbf{Y}
\right].
\]
On \(\{m_{w,\mathbf{x}}(\mathbf{Y})>0\}\),
\begin{align*}
\mathbf{u}_{w,\mathbf{x}}(\mathbf{Y})
-
\mathbf{u}_{\mathbf{x}}(\mathbf{Y})
&=
\frac{
\mathbb{E}_{\mathbf{x}}
\!\left[
w_{\mathbf{x}}(\mathbf{S}_1)\boldsymbol{\Delta}
\mid
\mathbf{Y}
\right]
}{
\mathbb{E}_{\mathbf{x}}
\!\left[
w_{\mathbf{x}}(\mathbf{S}_1)
\mid
\mathbf{Y}
\right]
}
-
\mathbb{E}_{\mathbf{x}}
\!\left[
\boldsymbol{\Delta}
\mid
\mathbf{Y}
\right] \\
&=
\frac{
\mathbb{E}_{\mathbf{x}}
\!\left[
w_{\mathbf{x}}(\mathbf{S}_1)\boldsymbol{\Delta}
\mid
\mathbf{Y}
\right]
-
\mathbb{E}_{\mathbf{x}}
\!\left[
w_{\mathbf{x}}(\mathbf{S}_1)
\mid
\mathbf{Y}
\right]
\mathbb{E}_{\mathbf{x}}
\!\left[
\boldsymbol{\Delta}
\mid
\mathbf{Y}
\right]
}{
m_{w,\mathbf{x}}(\mathbf{Y})
} \\
&=
\frac{
\operatorname{Cov}_{\mathbf{x}}
\!\left(
w_{\mathbf{x}}(\mathbf{S}_1),
\boldsymbol{\Delta}
\mid
\mathbf{Y}
\right)
}{
m_{w,\mathbf{x}}(\mathbf{Y})
}.
\end{align*}
This proves the theorem. \hfill\(\square\)

\subsection{Proof of Theorem~\ref{thm:DW-FM-regret}}

Fix a context \(\mathbf{x}\). For notational simplicity, write
\(C_{\widehat w,\mathbf{x}}:=A_{\mathbf{x}}B_{\mathbf{x}}\). By
Assumption~\ref{ass:DW-FM-coverage},
\[
M_{\mathbf{x}}(t,\mathbf{s})
\rho_{\mathbf{x},t}^{\theta}(\mathbf{s})
\le
C_{\widehat w,\mathbf{x}}\,
m_{\widehat w,\mathbf{x}}(t,\mathbf{s})
p_{\mathbf{x}}(t,\mathbf{s})
\]
for almost every \((t,\mathbf{s})\). Therefore,
\begin{align*}
L_{\mathrm{ideal},\mathbf{x}}(\theta)
&=
\int_0^1
\int_{\mathcal S}
M_{\mathbf{x}}(t,\mathbf{s})
\left\|
\mathbf{v}_\theta(t,\mathbf{s},\mathbf{x})
-
\mathbf{u}_{\mathbf{x}}(t,\mathbf{s})
\right\|^2
\,dq_{\mathbf{x},t}^{\theta}(\mathbf{s})\,dt \\
&\le
C_{\widehat w,\mathbf{x}}
\int_0^1
\int_{\mathcal S}
m_{\widehat w,\mathbf{x}}(t,\mathbf{s})
\left\|
\mathbf{v}_\theta(t,\mathbf{s},\mathbf{x})
-
\mathbf{u}_{\mathbf{x}}(t,\mathbf{s})
\right\|^2
p_{\mathbf{x}}(t,\mathbf{s})\,d\mathbf{s}\,dt .
\end{align*}
Here \(m_{\widehat w,\mathbf{x}}(t,\mathbf{s})\) and
\(\mathbf u_{\widehat w,\mathbf{x}}(t,\mathbf{s})\) abbreviate the corresponding
quantities evaluated at \(\mathbf Y=(t,\mathbf{s},\mathbf{x})\).

Write
\[
\mathbf{a}(t,\mathbf{s})
:=
\mathbf{v}_\theta(t,\mathbf{s},\mathbf{x})
-
\mathbf{u}_{\widehat w,\mathbf{x}}(t,\mathbf{s}),
\qquad
\mathbf{b}(t,\mathbf{s})
:=
\mathbf{u}_{\widehat w,\mathbf{x}}(t,\mathbf{s})
-
\mathbf{u}_{\mathbf{x}}(t,\mathbf{s}).
\]
Then
\[
\left\|
\mathbf{v}_\theta(t,\mathbf{s},\mathbf{x})
-
\mathbf{u}_{\mathbf{x}}(t,\mathbf{s})
\right\|^2
=
\|\mathbf{a}(t,\mathbf{s})+\mathbf{b}(t,\mathbf{s})\|^2
\le
2\|\mathbf{a}(t,\mathbf{s})\|^2
+
2\|\mathbf{b}(t,\mathbf{s})\|^2.
\]
Therefore,
\begin{align*}
L_{\mathrm{ideal},\mathbf{x}}(\theta)
&\le
2C_{\widehat w,\mathbf{x}}
\int_0^1
\int_{\mathcal S}
m_{\widehat w,\mathbf{x}}(t,\mathbf{s})
\left\|
\mathbf{v}_\theta(t,\mathbf{s},\mathbf{x})
-
\mathbf{u}_{\widehat w,\mathbf{x}}(t,\mathbf{s})
\right\|^2
p_{\mathbf{x}}(t,\mathbf{s})\,d\mathbf{s}\,dt \\
&\quad+
2C_{\widehat w,\mathbf{x}}
\int_0^1
\int_{\mathcal S}
m_{\widehat w,\mathbf{x}}(t,\mathbf{s})
\left\|
\mathbf{u}_{\widehat w,\mathbf{x}}(t,\mathbf{s})
-
\mathbf{u}_{\mathbf{x}}(t,\mathbf{s})
\right\|^2
p_{\mathbf{x}}(t,\mathbf{s})\,d\mathbf{s}\,dt .
\end{align*}
By the weighted least-squares identity~\eqref{eq: weighted loss identity}
established in the proof of Theorem~\ref{thm:endpoint-population-target},
\[
\int_0^1
\int_{\mathcal S}
m_{\widehat w,\mathbf{x}}(t,\mathbf{s})
\left\|
\mathbf{v}_\theta(t,\mathbf{s},\mathbf{x})
-
\mathbf{u}_{\widehat w,\mathbf{x}}(t,\mathbf{s})
\right\|^2
p_{\mathbf{x}}(t,\mathbf{s})\,d\mathbf{s}\,dt
=
E_{\widehat w,\mathbf{x}}(\mathbf{v}_\theta),
\]
while the second integral is exactly
\(B_{\mathrm{tilt},\mathbf{x}}(\widehat w)\). Hence
\[
L_{\mathrm{ideal},\mathbf{x}}(\theta)
\le
2C_{\widehat w,\mathbf{x}}
\left(
E_{\widehat w,\mathbf{x}}(\mathbf{v}_\theta)
+
B_{\mathrm{tilt},\mathbf{x}}(\widehat w)
\right).
\]
The pointwise regret bound from Theorem~\ref{thm:ideal-controls-regret} gives
\begin{equation}
\mathrm{Reg}_{\mathbf{x}}(\theta)
\le
2L_{\mathrm{ideal},\mathbf{x}}(\theta)^{1/2}
\le
2\sqrt{
2C_{\widehat w,\mathbf{x}}
\left(
E_{\widehat w,\mathbf{x}}(\mathbf{v}_\theta)
+
B_{\mathrm{tilt},\mathbf{x}}(\widehat w)
\right)
}.
\label{eq:pointwise-regret}
\end{equation}
Averaging over \(\mathbf X\) yields the expected-regret bound.
\(\square\)
\subsection{Stability of the Plug-in Endpoint Score}
\label{app:plugin-weight-stability}

This appendix quantifies the error incurred when the oracle decision
\(\mathbf z_{\mathbf x}^\star\) in the endpoint score is replaced by the
reference decision \(\widehat{\mathbf z}_{\mathbf x}\). Recall that
\[
w_{\mathbf x}^\star(\mathbf s)
=
1+\lambda
\left\|
\nabla_{\mathbf s}
\ell_{\mathbf x}(\mathbf z_{\mathbf x}^\star,\mathbf s)
\right\|^2,
\qquad
\widehat w_{\mathbf x}(\mathbf s)
=
1+\lambda
\left\|
\nabla_{\mathbf s}
\ell_{\mathbf x}(\widehat{\mathbf z}_{\mathbf x},\mathbf s)
\right\|^2 .
\]

\begin{lemma}
\label{lem:plugin-weight-stability}
Fix a context \(\mathbf x\). Suppose there exist constants
\(L_{\mathbf x}<\infty\) and \(G_{\mathbf x}<\infty\) such that, for all
relevant \(\mathbf s\in\mathcal S\) and all relevant decisions
\(\mathbf z,\mathbf z'\in\mathcal Z\),
\[
\left\|
\nabla_{\mathbf s}\ell_{\mathbf x}(\mathbf z,\mathbf s)
-
\nabla_{\mathbf s}\ell_{\mathbf x}(\mathbf z',\mathbf s)
\right\|
\le
L_{\mathbf x}
\|\mathbf z-\mathbf z'\|,
\]
and
\[
\left\|
\nabla_{\mathbf s}
\ell_{\mathbf x}(\mathbf z_{\mathbf x}^\star,\mathbf s)
\right\|
\le
G_{\mathbf x}.
\]
If
\[
\|\widehat{\mathbf z}_{\mathbf x}-\mathbf z_{\mathbf x}^\star\|
\le
\varepsilon_{\mathbf x},
\]
then
\[
\sup_{\mathbf s\in\mathcal S}
\left|
\widehat w_{\mathbf x}(\mathbf s)
-
w_{\mathbf x}^\star(\mathbf s)
\right|
\le
\lambda
\left(
2G_{\mathbf x}L_{\mathbf x}\varepsilon_{\mathbf x}
+
L_{\mathbf x}^2\varepsilon_{\mathbf x}^2
\right).
\]
\end{lemma}

\begin{proof}
Fix \(\mathbf s\in\mathcal S\) and write
\[
\mathbf a
:=
\nabla_{\mathbf s}
\ell_{\mathbf x}(\widehat{\mathbf z}_{\mathbf x},\mathbf s),
\qquad
\mathbf b
:=
\nabla_{\mathbf s}
\ell_{\mathbf x}(\mathbf z_{\mathbf x}^\star,\mathbf s).
\]
Then
\[
\left|
\widehat w_{\mathbf x}(\mathbf s)
-
w_{\mathbf x}^\star(\mathbf s)
\right|
=
\lambda
\left|
\|\mathbf a\|^2-\|\mathbf b\|^2
\right|.
\]
Using
\[
\left|
\|\mathbf a\|^2-\|\mathbf b\|^2
\right|
\le
\|\mathbf a-\mathbf b\|
\left(
\|\mathbf a\|+\|\mathbf b\|
\right),
\]
the Lipschitz condition gives
\[
\|\mathbf a-\mathbf b\|
\le
L_{\mathbf x}
\|\widehat{\mathbf z}_{\mathbf x}-\mathbf z_{\mathbf x}^\star\|
\le
L_{\mathbf x}\varepsilon_{\mathbf x}.
\]
Moreover,
\[
\|\mathbf a\|
\le
\|\mathbf b\|+\|\mathbf a-\mathbf b\|
\le
G_{\mathbf x}
+
L_{\mathbf x}\varepsilon_{\mathbf x},
\qquad
\|\mathbf b\|\le G_{\mathbf x}.
\]
Therefore,
\[
\left|
\widehat w_{\mathbf x}(\mathbf s)
-
w_{\mathbf x}^\star(\mathbf s)
\right|
\le
\lambda
L_{\mathbf x}\varepsilon_{\mathbf x}
\left(
2G_{\mathbf x}
+
L_{\mathbf x}\varepsilon_{\mathbf x}
\right),
\]
which is the desired bound. Taking the supremum over \(\mathbf s\) completes the
proof.
\end{proof}

Lemma~\ref{lem:plugin-weight-stability} shows that the plug-in endpoint score
converges to the oracle endpoint score as the reference decision improves. In
particular, for fixed \(\lambda\), the score error is first order in
\(\|\widehat{\mathbf z}_{\mathbf x}-\mathbf z_{\mathbf x}^\star\|\) when the
reference decision error is small.
\section{When Can Decision Weighting Improve over Ordinary FM}
\label{app: misspecification}

Section \ref{sec:theory} shows that endpoint weighting changes the population target of the
FM regression problem. This appendix explains when such a change can be beneficial. The point is not that weighting improves a fully
realizable regression problem. If the vector-field class can represent the target
velocity exactly, then ordinary FM and weighted FM can both attain zero approximation
error. The benefit appears when the vector-field class has limited capacity and
cannot fit all path regions equally well.

In this regime, ordinary FM weights errors according to how frequently path locations are
sampled by the training interpolation distribution. As a result, it can prefer a
model that fits high-probability path regions even if those regions have little
downstream decision impact. A decision-weighted objective can improve over ordinary
FM when decision-critical regions are underweighted by the ordinary path measure and
the decision-sensitive weight assigns sufficiently larger relative weight to those
regions.

We now formalize this mechanism in a fixed-context path-space notation. Fix a context
\(\mathbf{x}\), and let
\[
\Omega
:=
[0,1]\times\calS
\]
be the path-location space. A point in this space is written as
\((t,\mathbf{s})\), representing an interpolation time and an interpolation state.
Let \(\nu_{\mathbf{x}}\) denote the ordinary FM path measure, namely the conditional
law of \((T,\mathbf{S}_T)\) given \(\mathbf{X}=\mathbf{x}\) under the standard
interpolation sampling scheme. Equivalently, if \(P\) denotes the population law of
the full FM regression tuple
\[
(\mathbf{Y},\boldsymbol{\Delta},\mathbf{S}_1),
\qquad
\mathbf{Y}=(T,\mathbf{S}_T,\mathbf{X}),
\]
then \(\nu_{\mathbf{x}}\) is the conditional marginal of \(P\) on
\((T,\mathbf{S}_T)\) given \(\mathbf{X}=\mathbf{x}\). Thus, for any measurable
function \(g\) on \(\Omega\),
\[
\int_{\Omega}
g(t,\mathbf{s})\,d\nu_{\mathbf{x}}(t,\mathbf{s})
=
\E
\!\left[
g(T,\mathbf{S}_T)
\mid
\mathbf{X}=\mathbf{x}
\right].
\]
In particular, for a path region \(A\subseteq\Omega\),
\(\nu_{\mathbf{x}}(A)\) is the probability that ordinary FM samples a regression
point from \(A\).

For a target velocity \(\mathbf{u}_{\mathbf{x}}\), define the ordinary FM
approximation risk
\[
L_{\mathrm{FM},\mathbf{x}}^{\nu}(\mathbf v)
:=
\int_{\Omega}
\left\|
\mathbf v(t,\mathbf{s},\mathbf{x})
-
\mathbf u_{\mathbf{x}}(t,\mathbf{s})
\right\|^2
\,d\nu_{\mathbf{x}}(t,\mathbf{s}).
\]
For a nonnegative path-space weight \(\alpha_{\mathbf{x}}\), define the corresponding
weighted approximation risk
\[
L_{\alpha,\mathbf{x}}^{\nu}(\mathbf v)
:=
\int_{\Omega}
\alpha_{\mathbf{x}}(t,\mathbf{s})
\left\|
\mathbf v(t,\mathbf{s},\mathbf{x})
-
\mathbf u_{\mathbf{x}}(t,\mathbf{s})
\right\|^2
\,d\nu_{\mathbf{x}}(t,\mathbf{s}).
\]
The generic weight \(\alpha_{\mathbf{x}}\) is an analytical device that describes how
much emphasis a weighted surrogate places on each path location. We use a generic
\(\alpha_{\mathbf{x}}\), rather than committing to a particular weight, because the
purpose of this example is only to isolate the approximation-allocation mechanism:
under misspecification, any path-level weight that sufficiently emphasizes
decision-critical regions can change which model is selected. The ideal adjoint
weight \(M_{\mathbf{x}}\) and the endpoint-induced effective weight
\(m_{\widehat w,\mathbf{x}}\) should be viewed as two possible instantiations of this
generic path-space weighting principle.

\begin{theorem}
\label{thm:separation}
For every $\varepsilon\in(0,1/2)$, there exist a two-region path space
$\Omega=A\cup B$ with
\[
\nu(A)=\varepsilon,
\qquad
\nu(B)=1-\varepsilon,
\]
a target velocity $\mathbf u$, an ideal decision weight $M$, and a two-element
model class
\[
\calV
=
\{
\mathbf v_{\mathrm{good}},
\mathbf v_{\mathrm{bad}}
\}
\]
such that:
\begin{enumerate}[leftmargin=1.5em,itemsep=0.15em]
\item ordinary unweighted FM risk selects $\mathbf v_{\mathrm{bad}}$;
\item the ideal decision-weighted risk selects $\mathbf v_{\mathrm{good}}$; and
\item the ideal weighted-risk gap satisfies
\[
\frac{
L_{\alpha,\mathbf{x}}^{\nu}(\mathbf v_{\mathrm{bad}})
}{
L_{\alpha,\mathbf{x}}^{\nu}(\mathbf v_{\mathrm{good}})
}
\ge
\frac{1}{2\varepsilon^2}.
\]
\end{enumerate}
Hence the gap between the model selected by ordinary FM and the model selected by
decision-weighted FM can be made arbitrarily large as $\varepsilon\downarrow 0$.
\end{theorem}

\begin{proof}
Let \(A\) and \(B\) be disjoint path regions with
\[
\nu_{\mathbf{x}}(A)=\varepsilon,
\qquad
\nu_{\mathbf{x}}(B)=1-\varepsilon,
\qquad
\varepsilon\in(0,1/2).
\]
Region \(A\) should be interpreted as low-probability under the ordinary FM path
measure but decision-critical, while region \(B\) is high-probability but less
decision-sensitive.

Construct a two-element model class whose squared pathwise errors satisfy
\[
\left\|
\mathbf v_{\mathrm{good}}(t,\mathbf{s},\mathbf{x})
-
\mathbf u_{\mathbf{x}}(t,\mathbf{s})
\right\|^2
=
\begin{cases}
0, & (t,\mathbf{s})\in A,\\[0.2em]
2\varepsilon, & (t,\mathbf{s})\in B,
\end{cases}
\]
and
\[
\left\|
\mathbf v_{\mathrm{bad}}(t,\mathbf{s},\mathbf{x})
-
\mathbf u_{\mathbf{x}}(t,\mathbf{s})
\right\|^2
=
\begin{cases}
1, & (t,\mathbf{s})\in A,\\[0.2em]
0, & (t,\mathbf{s})\in B.
\end{cases}
\]
This construction can be realized, for example, with scalar velocities by taking
\(\mathbf u_{\mathbf{x}}=0\) and choosing piecewise constant candidate fields with
the displayed squared errors.

The ordinary FM approximation risks are
\begin{align*}
L_{\mathrm{FM},\mathbf{x}}^{\nu}(\mathbf v_{\mathrm{good}})
&=
0\cdot\varepsilon
+
2\varepsilon(1-\varepsilon),
\\
L_{\mathrm{FM},\mathbf{x}}^{\nu}(\mathbf v_{\mathrm{bad}})
&=
1\cdot\varepsilon
+
0\cdot(1-\varepsilon)
=
\varepsilon.
\end{align*}
Since \(\varepsilon<1/2\), we have
\[
\varepsilon
<
2\varepsilon(1-\varepsilon).
\]
Therefore ordinary FM selects \(\mathbf v_{\mathrm{bad}}\).

Now consider the weighted approximation risk. Since
\(\alpha_{\mathbf{x}}=\alpha_A\) on \(A\) and
\(\alpha_{\mathbf{x}}=\alpha_B\) on \(B\),
\begin{align*}
L_{\alpha,\mathbf{x}}^{\nu}(\mathbf v_{\mathrm{good}})
&=
\alpha_A\cdot 0\cdot\varepsilon
+
\alpha_B\cdot 2\varepsilon(1-\varepsilon)
=
2\alpha_B\varepsilon(1-\varepsilon),
\\
L_{\alpha,\mathbf{x}}^{\nu}(\mathbf v_{\mathrm{bad}})
&=
\alpha_A\cdot 1\cdot\varepsilon
+
\alpha_B\cdot 0\cdot(1-\varepsilon)
=
\alpha_A\varepsilon.
\end{align*}
The weighted objective selects \(\mathbf v_{\mathrm{good}}\) whenever
\[
2\alpha_B\varepsilon(1-\varepsilon)
<
\alpha_A\varepsilon,
\]
which is equivalent to
\[
\frac{\alpha_A}{\alpha_B}
>
2(1-\varepsilon).
\]
Finally, if \(\alpha_A/\alpha_B=\varepsilon^{-2}\), then
\[
\frac{
L_{\alpha,\mathbf{x}}^{\nu}(\mathbf v_{\mathrm{bad}})
}{
L_{\alpha,\mathbf{x}}^{\nu}(\mathbf v_{\mathrm{good}})
}
=
\frac{
\alpha_A\varepsilon
}{
2\alpha_B\varepsilon(1-\varepsilon)
}
=
\frac{\alpha_A}{2\alpha_B(1-\varepsilon)}
=
\frac{1}{2\varepsilon^2(1-\varepsilon)}
\ge
\frac{1}{2\varepsilon^2}.
\]
This proves the theorem.
\end{proof}

Theorem~\ref{thm:separation} should be read as an approximation-allocation result,
not as an unconditional guarantee that every decision weight improves over ordinary
FM. It identifies a regime in which weighting can help: the model class is
misspecified, the ordinary FM path measure assigns small mass to a decision-critical
region, and the path-level weight assigns that region sufficiently larger relative
weight. For DW-FM, the relevant path-level weight is
\(m_{\widehat w,\mathbf{x}}\). Therefore practical endpoint weighting benefits from
this mechanism when the endpoint proxy induces a path weight that is aligned with
decision-critical regions and when the tilting bias remains controlled.

Thus decision weighting does not make the vector-field class more expressive. Rather,
under misspecification, it changes how limited approximation capacity is allocated
across the transport path. Ordinary FM allocates capacity according to path-sampling
frequency, while decision-weighted FM can allocate capacity according to downstream
decision relevance when the induced path weight is well aligned with regret-sensitive
regions.

\section{Improved Regret for Strongly Convex Downstream Problems}
\label{app:strongly-convex-first-order}

The regret analysis in Section \ref{sec:theory} controls a uniform decision discrepancy over the entire
decision class. This route is deliberately general as it protects against terminal risk errors
for all decisions $\mathbf{z}\in\calZ$. However, this uniform-discrepancy route can be
conservative in some scenarios. In this section, we consider the case that the downstream risk is smooth and strongly convex in the decision variable and propose a sharper curvature-based refinement on the regret bound.

We first state the following required conditions.

\begin{assumption}[Downstream curvature and unconstrained optimality]
\label{ass:first-order-curvature}
Fix a context $\mathbf{x}$ and suppose $\calZ=\R^m$. The true risk
$R_{\mathbf{x}}(\cdot;q_{\mathbf{x}}^\star)$ is differentiable and
$L_{z,\mathbf{x}}$-smooth. The learned-law risk
$R_{\mathbf{x}}(\cdot;q_{\theta,\mathbf{x}})$ is differentiable and
$\mu_{\mathbf{x}}$-strongly convex, and admits a minimizer
\[
\mathbf{z}_{\theta}(\mathbf{x})
\in
\argmin_{\mathbf{z}\in\R^m}
R_{\mathbf{x}}(\mathbf{z};q_{\theta,\mathbf{x}}).
\]
Moreover, the population optimizer
\[
\mathbf{z}_{\mathbf{x}}^\star
\in
\argmin_{\mathbf{z}\in\R^m}
R_{\mathbf{x}}(\mathbf{z};q_{\mathbf{x}}^\star)
\]
is an unconstrained minimizer.
\end{assumption}

Assumption~\ref{ass:first-order-curvature} is a local curvature condition on the
downstream optimization problem. The smoothness of
$R_{\mathbf{x}}(\cdot;q_{\mathbf{x}}^\star)$ says that, under the true law, regret
grows at most quadratically with the distance from the population optimizer
$\mathbf{z}_{\mathbf{x}}^\star$. The strong convexity of
$R_{\mathbf{x}}(\cdot;q_{\theta,\mathbf{x}})$ says that, under the learned law, a
small violation of the first-order optimality condition implies that the plug-in
optimizer $\mathbf{z}_{\theta}(\mathbf{x})$ must be close to
$\mathbf{z}_{\mathbf{x}}^\star$. This condition holds in many regularized stochastic optimization problems. The unconstrained assumption
$\calZ=\R^m$ is used to write the optimality conditions as ordinary gradient
equalities,
\[
\nabla_{\mathbf{z}}
R_{\mathbf{x}}(\mathbf{z}_{\theta}(\mathbf{x});q_{\theta,\mathbf{x}})
=
\mathbf{0},
\qquad
\nabla_{\mathbf{z}}
R_{\mathbf{x}}(\mathbf{z}_{\mathbf{x}}^\star;q_{\mathbf{x}}^\star)
=
\mathbf{0}.
\]

\begin{assumption}[Gradient interchange]
\label{ass:first-order-gradient-interchange}
For $q=q_{\mathbf{x}}^\star$ and $q=q_{\theta,\mathbf{x}}$, differentiation can be
interchanged with integration:
\[
\nabla_{\mathbf{z}}R_{\mathbf{x}}(\mathbf{z};q)
=
\E_{\mathbf{S}\sim q}
\!\left[
\nabla_{\mathbf{z}}\ell_{\mathbf{x}}(\mathbf{z},\mathbf{S})
\right].
\]
\end{assumption}

Assumption~\ref{ass:first-order-gradient-interchange} is a regularity condition
ensuring that the first-order condition of the downstream risk can be written as an
expectation of sample-level gradients. It holds, for example, under standard
dominated-convergence conditions on
$\nabla_{\mathbf{z}}\ell_{\mathbf{x}}(\mathbf{z},\mathbf{S})$.

\begin{assumption}[First-order adjoint regularity]
\label{ass:first-order-adjoint}
The vector-valued terminal function
\[
\mathbf{s}
\mapsto
\nabla_{\mathbf{z}}\ell_{\mathbf{x}}
(\mathbf{z}_{\mathbf{x}}^\star,\mathbf{s})
\in\R^m
\]
admits a sufficiently regular backward-transport solution
$\psi_{\mathbf{x}}(t,\mathbf{s})\in\R^m$ satisfying
\begin{equation}
\partial_t\psi_{\mathbf{x}}(t,\mathbf{s})
+
\nabla_{\mathbf{s}}\psi_{\mathbf{x}}(t,\mathbf{s})\,
\mathbf{u}_{\mathbf{x}}(t,\mathbf{s})
=
\mathbf{0},
\qquad
\psi_{\mathbf{x}}(1,\mathbf{s})
=
\nabla_{\mathbf{z}}\ell_{\mathbf{x}}
(\mathbf{z}_{\mathbf{x}}^\star,\mathbf{s}).
\label{eq:first-order-adjoint-pde}
\end{equation}
Moreover, the same integrability and differentiability conditions required in
Proposition~\ref{prop:adjoint-identity} hold for each coordinate of
$\psi_{\mathbf{x}}$, and the same no-boundary-flux condition as in
Assumption~\ref{ass:no-boundary-flux} holds for each coordinate of
$\psi_{\mathbf{x}}$ along both the ideal and learned paths.
\end{assumption}

Assumption~\ref{ass:first-order-adjoint} is the first-order analogue of the scalar
adjoint regularity used in the main pathwise regret analysis. In Section \ref{sec:method}, the
scalar adjoint $\phi_{\mathbf{x},\mathbf{z}}$ transports the terminal loss
$\ell_{\mathbf{x}}(\mathbf{z},\mathbf{s})$ backward along the ideal flow. Here,
because downstream strong convexity lets us control regret through the first-order
optimality condition at $\mathbf{z}_{\mathbf{x}}^\star$, we instead transport the
terminal gradient
$\nabla_{\mathbf{z}}\ell_{\mathbf{x}}
(\mathbf{z}_{\mathbf{x}}^\star,\mathbf{s})$
backward along the same ideal flow. The resulting adjoint is
$\psi_{\mathbf{x}}$.

The role of $\psi_{\mathbf{x}}$ is to convert the terminal first-order residual
\[
\mathbf{g}_{\theta}
=
\E_{q_{\theta,\mathbf{x}}}
\!\left[
\nabla_{\mathbf{z}}\ell_{\mathbf{x}}
(\mathbf{z}_{\mathbf{x}}^\star,\mathbf{S})
\right]
-
\E_{q_{\mathbf{x}}^\star}
\!\left[
\nabla_{\mathbf{z}}\ell_{\mathbf{x}}
(\mathbf{z}_{\mathbf{x}}^\star,\mathbf{S})
\right]
\]
into a pathwise velocity-error quantity. Under
Assumption~\ref{ass:first-order-adjoint}, applying the adjoint identity coordinate by
coordinate yields
\[
\mathbf{g}_{\theta}
=
\int_0^1
\int_{\calS}
\nabla_{\mathbf{s}}\psi_{\mathbf{x}}(t,\mathbf{s})
\left(
\mathbf{v}_\theta(t,\mathbf{s},\mathbf{x})
-
\mathbf{u}_{\mathbf{x}}(t,\mathbf{s})
\right)
q^\theta_{\mathbf{x},t}(\mathbf{s})
\,d\mathbf{s}\,dt.
\]
Thus $\psi_{\mathbf{x}}$ is introduced to connect the first-order regret control to
the flow-matching pathwise velocity error.

Based on the above conditions, we can obtain the following theorem.

\begin{theorem}
\label{thm:first-order}
Under Assumptions~\ref{ass:first-order-curvature},
\ref{ass:first-order-gradient-interchange}, and
\ref{ass:first-order-adjoint}, the following two bounds hold:
\begin{equation}
\Reg_{\mathbf{x}}(\theta)
\le
\frac{L_{z,\mathbf{x}}}{2\mu_{\mathbf{x}}^2}
\left\|
\E_{q_{\theta,\mathbf{x}}}
\!\left[
\nabla_{\mathbf{z}}\ell_{\mathbf{x}}
(\mathbf{z}_{\mathbf{x}}^\star,\mathbf{S})
\right]
-
\E_{q_{\mathbf{x}}^\star}
\!\left[
\nabla_{\mathbf{z}}\ell_{\mathbf{x}}
(\mathbf{z}_{\mathbf{x}}^\star,\mathbf{S})
\right]
\right\|^2 ,
\label{eq:first-order-regret-main}
\end{equation}
and
\begin{equation}
\Reg_{\mathbf{x}}(\theta)
\le
\frac{L_{z,\mathbf{x}}}{2\mu_{\mathbf{x}}^2}
\int_0^1
\int_{\calS}
\left\|
\nabla_{\mathbf{s}}\psi_{\mathbf{x}}(t,\mathbf{s})
\right\|_F^2
\left\|
\mathbf{v}_\theta(t,\mathbf{s},\mathbf{x})
-
\mathbf{u}_{\mathbf{x}}(t,\mathbf{s})
\right\|^2
q^\theta_{\mathbf{x},t}(\mathbf{s})\,d\mathbf{s}\,dt .
\notag
\end{equation}
\end{theorem}

Theorem~\ref{thm:first-order} should be interpreted as a curvature-based
sharpening of Theorem~\ref{thm:ideal-controls-regret}.
 The main improvement is in how pathwise velocity error is converted into regret.
The general bound in Theorem~\ref{thm:ideal-controls-regret} controls regret through
the square root of a weighted squared velocity error. In contrast,
Theorem~\ref{thm:first-order} uses downstream strong convexity to control regret
linearly by a first-order weighted squared velocity error. Thus the improvement comes
from replacing a square-root regret conversion with a linear one.
To illustrate the difference, suppose the relevant weighted squared velocity error is
of order $\varepsilon$. The general uniform-discrepancy bound gives a regret
bound of order $\sqrt{\varepsilon}$. In the first-order strong-convexity bound, the regret bound is of
order $\varepsilon$.

The refinement is also more local in the decision variable. In the general theorem,
the envelope
$M_{\mathbf{x}}(t,\mathbf{s})$
must protect against terminal loss errors for all feasible decisions
$\mathbf{z}\in\calZ$. In the first-order theorem, the weight
$\left\|
\nabla_{\mathbf{s}}\psi_{\mathbf{x}}(t,\mathbf{s})
\right\|_F^2$
measures the pathwise sensitivity of the first-order optimality condition at the
single decision $\mathbf{z}_{\mathbf{x}}^\star$. Hence the first-order refinement
replaces a global worst-case sensitivity over the whole decision class with a local
sensitivity of the optimality condition at the population optimizer.

\begin{proof}
Let
\[
\mathbf{g}_\theta
:=
\E_{q_{\theta,\mathbf{x}}}
\!\left[
\nabla_{\mathbf{z}}\ell_{\mathbf{x}}
(\mathbf{z}_{\mathbf{x}}^\star,\mathbf{S})
\right]
-
\E_{q_{\mathbf{x}}^\star}
\!\left[
\nabla_{\mathbf{z}}\ell_{\mathbf{x}}
(\mathbf{z}_{\mathbf{x}}^\star,\mathbf{S})
\right].
\]
By Assumption~\ref{ass:first-order-curvature},
$\mathbf{z}_{\mathbf{x}}^\star$ is an unconstrained minimizer of
$R_{\mathbf{x}}(\cdot;q_{\mathbf{x}}^\star)$, so
\[
\nabla_{\mathbf{z}}
R_{\mathbf{x}}(\mathbf{z}_{\mathbf{x}}^\star;q_{\mathbf{x}}^\star)
=
\mathbf{0}.
\]
By Assumption~\ref{ass:first-order-gradient-interchange},
\[
\mathbf{g}_\theta
=
\nabla_{\mathbf{z}}
R_{\mathbf{x}}(\mathbf{z}_{\mathbf{x}}^\star;q_{\theta,\mathbf{x}}).
\]

Let $\mathbf{z}_\theta:=\mathbf{z}_{\theta}(\mathbf{x})$. Since
$R_{\mathbf{x}}(\cdot;q_{\theta,\mathbf{x}})$ is differentiable and
$\mu_{\mathbf{x}}$-strongly convex, its gradient is
$\mu_{\mathbf{x}}$-strongly monotone. Therefore,
\[
\left\langle
\nabla_{\mathbf{z}}R_{\mathbf{x}}(\mathbf{z}_\theta;q_{\theta,\mathbf{x}})
-
\nabla_{\mathbf{z}}R_{\mathbf{x}}
(\mathbf{z}_{\mathbf{x}}^\star;q_{\theta,\mathbf{x}}),
\mathbf{z}_\theta-\mathbf{z}_{\mathbf{x}}^\star
\right\rangle
\ge
\mu_{\mathbf{x}}
\left\|
\mathbf{z}_\theta-\mathbf{z}_{\mathbf{x}}^\star
\right\|^2 .
\]
By Cauchy--Schwarz,
\[
\left\|
\nabla_{\mathbf{z}}R_{\mathbf{x}}(\mathbf{z}_\theta;q_{\theta,\mathbf{x}})
-
\nabla_{\mathbf{z}}R_{\mathbf{x}}
(\mathbf{z}_{\mathbf{x}}^\star;q_{\theta,\mathbf{x}})
\right\|
\ge
\mu_{\mathbf{x}}
\left\|
\mathbf{z}_\theta-\mathbf{z}_{\mathbf{x}}^\star
\right\|.
\]
By Assumption~\ref{ass:first-order-curvature},
$\mathbf{z}_\theta$ minimizes
$R_{\mathbf{x}}(\cdot;q_{\theta,\mathbf{x}})$ over $\R^m$, so
\[
\nabla_{\mathbf{z}}
R_{\mathbf{x}}(\mathbf{z}_\theta;q_{\theta,\mathbf{x}})
=
\mathbf{0}.
\]
It follows that
\[
\mu_{\mathbf{x}}
\left\|
\mathbf{z}_\theta-\mathbf{z}_{\mathbf{x}}^\star
\right\|
\le
\left\|
\mathbf{g}_\theta
\right\|.
\]

Next, by $L_{z,\mathbf{x}}$-smoothness of
$R_{\mathbf{x}}(\cdot;q_{\mathbf{x}}^\star)$ and first-order optimality of
$\mathbf{z}_{\mathbf{x}}^\star$,
\[
R_{\mathbf{x}}(\mathbf{z}_\theta;q_{\mathbf{x}}^\star)
-
R_{\mathbf{x}}(\mathbf{z}_{\mathbf{x}}^\star;q_{\mathbf{x}}^\star)
\le
\frac{L_{z,\mathbf{x}}}{2}
\left\|
\mathbf{z}_\theta-\mathbf{z}_{\mathbf{x}}^\star
\right\|^2.
\]
Combining the last two displays yields
\[
\Reg_{\mathbf{x}}(\theta)
\le
\frac{L_{z,\mathbf{x}}}{2\mu_{\mathbf{x}}^2}
\left\|
\mathbf{g}_\theta
\right\|^2,
\]
which proves \eqref{eq:first-order-regret-main}.

It remains to express $\mathbf{g}_\theta$ as a pathwise velocity error. By
Assumption~\ref{ass:first-order-adjoint}, we may apply
Proposition~\ref{prop:adjoint-identity} componentwise to the vector-valued terminal
function
\[
\mathbf{s}
\mapsto
\nabla_{\mathbf{z}}\ell_{\mathbf{x}}
(\mathbf{z}_{\mathbf{x}}^\star,\mathbf{s})
\in\R^m.
\]
This gives
\[
\mathbf{g}_\theta
=
\int_0^1
\int_{\calS}
\nabla_{\mathbf{s}}\psi_{\mathbf{x}}(t,\mathbf{s})
\left(
\mathbf{v}_\theta(t,\mathbf{s},\mathbf{x})
-
\mathbf{u}_{\mathbf{x}}(t,\mathbf{s})
\right)
q^\theta_{\mathbf{x},t}(\mathbf{s})\,d\mathbf{s}\,dt,
\]
where $\nabla_{\mathbf{s}}\psi_{\mathbf{x}}(t,\mathbf{s})$ is the
$m\times d$ Jacobian matrix. Hence
\begin{align*}
\left\|
\mathbf{g}_\theta
\right\|
&\le
\int_0^1
\int_{\calS}
\left\|
\nabla_{\mathbf{s}}\psi_{\mathbf{x}}(t,\mathbf{s})
\right\|_F
\left\|
\mathbf{v}_\theta(t,\mathbf{s},\mathbf{x})
-
\mathbf{u}_{\mathbf{x}}(t,\mathbf{s})
\right\|
q^\theta_{\mathbf{x},t}(\mathbf{s})\,d\mathbf{s}\,dt \\
&\le
\left(
\int_0^1
\int_{\calS}
\left\|
\nabla_{\mathbf{s}}\psi_{\mathbf{x}}(t,\mathbf{s})
\right\|_F^2
\left\|
\mathbf{v}_\theta(t,\mathbf{s},\mathbf{x})
-
\mathbf{u}_{\mathbf{x}}(t,\mathbf{s})
\right\|^2
q^\theta_{\mathbf{x},t}(\mathbf{s})\,d\mathbf{s}\,dt
\right)^{1/2},
\end{align*}
where the final step uses Cauchy--Schwarz with respect to the probability measure
$q^\theta_{\mathbf{x},t}(\mathbf{s})\,d\mathbf{s}\,dt$ on
$[0,1]\times\calS$. Substituting this bound into
\eqref{eq:first-order-regret-main} completes the proof.
\end{proof}

This refinement is not needed for the general validity of DW-FM. Its role is to show
that the $O(n^{-1/4})$-type rate obtained from the general uniform-discrepancy route
is not intrinsic to all downstream stochastic optimization problems. Under downstream
strong convexity, the regret closure itself can be sharper; a finite-sample analysis
that controls the corresponding first-order pathwise surrogate may yield faster regret
rates.
\section{Finite-Sample Analysis for Empirical DW-FM}
\label{app:finite-sample}

Theorem~\ref{thm:DW-FM-regret} is a population closure for any fixed parameter
\(\theta\). We now apply it to the empirical DW-FM estimator trained from the
contextual dataset
\(\mathcal D=\{(\mathbf x_i,\mathbf s_i)\}_{i=1}^n\).

Throughout this appendix, \(\mathbb E[\cdot]\) without a subscript denotes full
expectation over all random elements appearing in the corresponding quantity.
In particular, for the final regret bound,
\(\mathbb E[\mathrm{Reg}_{\mathbf X}(\widehat\theta)]\) averages over the
training sample, the interpolation randomness used in empirical FM training,
any randomization in the reference routine or optimizer, and an independent test
context \(\mathbf X\). Whenever only part of the randomness is averaged out, we
use an explicit subscript.

The analysis treats \(\widehat w\) as fixed relative to the empirical FM tuples
used to optimize \(\theta\). This is immediate when the reference decision rule
is fixed in advance or estimated from an independent pilot sample. If the same
sample is used both to construct \(\widehat w\) and to train the vector field,
the bound should be read as applying after sample splitting or cross-fitting.
All population quantities involving \(\widehat w\) in this appendix are
understood conditionally on the realized weight rule.

Starting from the observed training data
\(\mathcal D=\{(\mathbf x_i,\mathbf s_i)\}_{i=1}^n\), the FM training procedure
constructs one regression tuple from each data pair as follows. For each \(i\),
draw \(\mathbf s_{0,i}\sim q_{0,\mathbf x_i}\) and
\(t_i\sim\mathrm{Unif}[0,1]\), independently, and set
\[
\mathbf s_{t_i,i}
=
(1-t_i)\mathbf s_{0,i}+t_i\mathbf s_i,
\qquad
\boldsymbol\Delta_i
=
\mathbf s_i-\mathbf s_{0,i},
\qquad
\mathbf Y_i
=
(t_i,\mathbf s_{t_i,i},\mathbf x_i).
\]
Here \(\mathbf s_i\) is the observed endpoint, corresponding to
\(\mathbf S_1\) in the population notation.

Let \(\mathcal V:=\{\mathbf v_\theta:\theta\in\Theta\}\) be the vector-field
class. The empirical DW-FM objective is
\[
\widehat L_{\widehat w,n}(\theta)
:=
\frac{1}{n}
\sum_{i=1}^n
\widehat w_{\mathbf x_i}(\mathbf s_i)
\left\|
\mathbf v_\theta(\mathbf Y_i)-\boldsymbol\Delta_i
\right\|^2.
\]
Let \(\widehat\theta\) be the parameter returned by empirical DW-FM training,
assumed to be an \(\eta\)-approximate minimizer:
\[
\widehat L_{\widehat w,n}(\widehat\theta)
\le
\inf_{\theta\in\Theta}
\widehat L_{\widehat w,n}(\theta)
+
\eta .
\]
Thus \(\mathbf v_{\widehat\theta}\) is the learned vector field.

Recall that
\[
E_{\widehat w,\mathbf x}(\mathbf v)
=
L_{\widehat w,\mathbf x}(\mathbf v)
-
L_{\widehat w,\mathbf x}(\mathbf u_{\widehat w,\mathbf x}).
\]
Its test-context average, conditional on the training procedure, is
\[
\bar E_{\widehat w}(\mathbf v)
:=
\mathbb E_{\mathbf X}
\!\left[
E_{\widehat w,\mathbf X}(\mathbf v)
\mid
\mathcal A_n
\right],
\]
where \(\mathcal A_n\) denotes the sigma-field generated by the training
procedure, the reference-decision routine, and any optimization randomness.
When \(\mathbf v=\mathbf v_{\widehat\theta}\), this is the population excess
risk of the trained field, averaged over a fresh test context while holding the
learned model fixed.

\begin{assumption}[Bounded empirical DW-FM losses]
\label{ass:bounded-DW-FM-loss}
There exist constants \(W,D<\infty\) such that
\[
0\le \widehat w_{\mathbf X}(\mathbf S_1)\le W,
\qquad
\left\|
\mathbf v(\mathbf Y)-\boldsymbol\Delta
\right\|
\le D
\quad
\text{for all }\mathbf v\in\mathcal V,
\]
almost surely under the joint interpolation law of
\((\mathbf X,\mathbf S_1,\mathbf S_0,T)\) and the randomness of the reference
rule, if any.
\end{assumption}

Assumption~\ref{ass:bounded-DW-FM-loss} is a bounded-loss condition for the
weighted velocity-regression problem. The upper bound on
\(\widehat w_{\mathbf X}(\mathbf S_1)\) prevents a small number of
high-sensitivity endpoints from dominating the empirical objective, while the
bound on \(\|\mathbf v(\mathbf Y)-\boldsymbol\Delta\|\) ensures that the
weighted squared losses are uniformly bounded by \(WD^2\). This condition is
mainly used to justify the uniform-convergence step for the empirical DW-FM
objective. It can be satisfied, for example, when the scenario space is compact,
the base distribution is truncated or has bounded support, the vector-field
class is norm-controlled, and the loss gradients used to construct
\(\widehat w\) are bounded.

Define the joint plug-in weighted loss class
\[
\mathcal F_{\widehat w}
:=
\left\{
(\mathbf Y,\boldsymbol\Delta,\mathbf S_1)
\mapsto
\widehat w_{\mathbf X}(\mathbf S_1)
\left\|
\mathbf v(\mathbf Y)-\boldsymbol\Delta
\right\|^2
:\,
\mathbf v\in\mathcal V
\right\},
\]
where \(\mathbf X\) is the context component of \(\mathbf Y\). Let
\(\mathfrak R_n(\mathcal F_{\widehat w})\) denote the full expected Rademacher
complexity
\[
\mathfrak R_n(\mathcal F_{\widehat w})
:=
\mathbb E
\!\left[
\mathbb E_{\boldsymbol\sigma}
\left[
\sup_{\mathbf v\in\mathcal V}
\frac{1}{n}
\sum_{i=1}^n
\sigma_i
\widehat w_{\mathbf X_i}(\mathbf S_{1,i})
\left\|
\mathbf v(\mathbf Y_i)-\boldsymbol\Delta_i
\right\|^2
\right]
\right],
\]
where \(\sigma_1,\ldots,\sigma_n\) are independent Rademacher signs.

\begin{theorem}
\label{thm:DW-FM-excess-risk}
Under Assumption~\ref{ass:bounded-DW-FM-loss},
\[
\mathbb E
\!\left[
\bar E_{\widehat w}(\mathbf v_{\widehat\theta})
\right]
\le
\mathbb E
\!\left[
\inf_{\theta\in\Theta}
\bar E_{\widehat w}(\mathbf v_\theta)
\right]
+
4\mathfrak R_n(\mathcal F_{\widehat w})
+
\eta .
\]
\end{theorem}

\begin{proof}
Let \(\mathbb P_{\widehat w}^{\mathrm{fm}}\) be the population interpolation
distribution of \((\mathbf Y,\boldsymbol\Delta,\mathbf S_1)\) associated with
the fixed weight rule \(\widehat w\), and let
\(\mathbb P_n^{\mathrm{fm}}\) be the empirical measure induced by the training
tuples. For each \(\theta\in\Theta\), define
\[
f_\theta(\mathbf Y,\boldsymbol\Delta,\mathbf S_1)
:=
\widehat w_{\mathbf X}(\mathbf S_1)
\left\|
\mathbf v_\theta(\mathbf Y)-\boldsymbol\Delta
\right\|^2 .
\]
By Assumption~\ref{ass:bounded-DW-FM-loss}, each \(f_\theta\) takes values in
\([0,WD^2]\). For any realization of the training procedure,
\begin{align*}
\mathbb P_{\widehat w}^{\mathrm{fm}} f_{\widehat\theta}
&\le
\mathbb P_n^{\mathrm{fm}} f_{\widehat\theta}
+
\sup_{\theta\in\Theta}
\left|
\mathbb P_{\widehat w}^{\mathrm{fm}} f_\theta
-
\mathbb P_n^{\mathrm{fm}} f_\theta
\right| \\
&\le
\inf_{\theta\in\Theta}
\mathbb P_n^{\mathrm{fm}} f_\theta
+
\eta
+
\sup_{\theta\in\Theta}
\left|
\mathbb P_{\widehat w}^{\mathrm{fm}} f_\theta
-
\mathbb P_n^{\mathrm{fm}} f_\theta
\right| \\
&\le
\inf_{\theta\in\Theta}
\mathbb P_{\widehat w}^{\mathrm{fm}} f_\theta
+
\eta
+
2
\sup_{\theta\in\Theta}
\left|
\mathbb P_{\widehat w}^{\mathrm{fm}} f_\theta
-
\mathbb P_n^{\mathrm{fm}} f_\theta
\right|.
\end{align*}
Taking full expectation and applying the standard expected Rademacher
uniform-deviation bound conditionally on the realized weight rule gives
\[
\mathbb E
\!\left[
\sup_{\theta\in\Theta}
\left|
\mathbb P_{\widehat w}^{\mathrm{fm}} f_\theta
-
\mathbb P_n^{\mathrm{fm}} f_\theta
\right|
\right]
\le
2\mathfrak R_n(\mathcal F_{\widehat w}).
\]
Therefore,
\[
\mathbb E
\!\left[
\mathbb P_{\widehat w}^{\mathrm{fm}} f_{\widehat\theta}
\right]
\le
\mathbb E
\!\left[
\inf_{\theta\in\Theta}
\mathbb P_{\widehat w}^{\mathrm{fm}} f_\theta
\right]
+
\eta
+
4\mathfrak R_n(\mathcal F_{\widehat w}).
\]
Finally,
\[
\mathbb P_{\widehat w}^{\mathrm{fm}} f_\theta
=
\mathbb E_{\mathbf X}
\!\left[
L_{\widehat w,\mathbf X}(\mathbf v_\theta)
\mid
\mathcal A_n
\right].
\]
Subtracting the context-averaged population constant
\[
\mathbb E_{\mathbf X}
\!\left[
L_{\widehat w,\mathbf X}(\mathbf u_{\widehat w,\mathbf X})
\mid
\mathcal A_n
\right],
\]
which does not depend on \(\theta\), yields the stated bound.
\end{proof}

We next control the averaged tilting bias. Define
\[
\bar B_{\mathrm{tilt}}(\widehat w)
:=
\mathbb E_{\mathbf X}
\!\left[
B_{\mathrm{tilt},\mathbf X}(\widehat w)
\mid
\mathcal A_n
\right].
\]
This is a population bias induced by endpoint weighting, not a sampling error.
Write
\[
\widehat w_{\mathbf X}(\mathbf s)
=
1+\lambda \widehat g_{\mathbf X}(\mathbf s),
\qquad
\widehat g_{\mathbf X}(\mathbf s)
:=
\left\|
\nabla_{\mathbf s}
\ell_{\mathbf X}(\widehat{\mathbf z}_{\mathbf X},\mathbf s)
\right\|^2 .
\]

\begin{assumption}[Bounded conditional covariance for tilting]
\label{ass:tilt-cov-plugin}
There exists \(C_{\mathrm{tilt}}<\infty\) such that
\[
\mathbb E
\!\left[
\left\|
\operatorname{Cov}
\!\left(
\widehat g_{\mathbf X}(\mathbf S_1),
\boldsymbol\Delta
\mid
\mathbf Y,\mathcal A_n
\right)
\right\|^2
\right]
\le
C_{\mathrm{tilt}}.
\]
\end{assumption}

Assumption~\ref{ass:tilt-cov-plugin} controls the amount by which endpoint
weighting changes the population FM target. It does not require the tilting
bias to vanish; it only requires the source of tilting to have a finite second
moment. Intuitively, the assumption rules out cases in which the
decision-sensitive weights are extremely correlated with rare, very large
velocity labels after conditioning on the interpolation point. It is mild when
the loss-gradient scores and interpolation velocities are bounded, and it also
holds under standard moment assumptions ensuring that this conditional
covariance is square integrable.

\begin{proposition}
\label{prop:tilting-bias-bound-plugin}
Under Assumption~\ref{ass:tilt-cov-plugin},
\[
\mathbb E
\!\left[
\bar B_{\mathrm{tilt}}(\widehat w)
\right]
\le
\lambda^2 C_{\mathrm{tilt}}.
\]
\end{proposition}

\begin{proof}
By Theorem~\ref{thm:endpoint-population-target}, applied conditionally on the
training procedure,
\[
\mathbf u_{\widehat w,\mathbf X}(\mathbf Y)
-
\mathbf u_{\mathbf X}(\mathbf Y)
=
\frac{
\operatorname{Cov}
\!\left(
\widehat w_{\mathbf X}(\mathbf S_1),
\boldsymbol\Delta
\mid
\mathbf Y,\mathcal A_n
\right)
}{
m_{\widehat w,\mathbf X}(\mathbf Y)
}.
\]
Substituting this identity into the definition of
\(B_{\mathrm{tilt},\mathbf X}(\widehat w)\) and then taking full expectation
gives
\[
\mathbb E
\!\left[
\bar B_{\mathrm{tilt}}(\widehat w)
\right]
=
\mathbb E
\!\left[
\frac{
\left\|
\operatorname{Cov}
\!\left(
\widehat w_{\mathbf X}(\mathbf S_1),
\boldsymbol\Delta
\mid
\mathbf Y,\mathcal A_n
\right)
\right\|^2
}{
m_{\widehat w,\mathbf X}(\mathbf Y)
}
\right].
\]
Since \(\widehat w_{\mathbf X}=1+\lambda \widehat g_{\mathbf X}\),
\[
\operatorname{Cov}
\!\left(
\widehat w_{\mathbf X}(\mathbf S_1),
\boldsymbol\Delta
\mid
\mathbf Y,\mathcal A_n
\right)
=
\lambda
\operatorname{Cov}
\!\left(
\widehat g_{\mathbf X}(\mathbf S_1),
\boldsymbol\Delta
\mid
\mathbf Y,\mathcal A_n
\right).
\]
Moreover, \(m_{\widehat w,\mathbf X}(\mathbf Y)\ge 1\). Hence
\[
\mathbb E
\!\left[
\bar B_{\mathrm{tilt}}(\widehat w)
\right]
\le
\lambda^2
\mathbb E
\!\left[
\left\|
\operatorname{Cov}
\!\left(
\widehat g_{\mathbf X}(\mathbf S_1),
\boldsymbol\Delta
\mid
\mathbf Y,\mathcal A_n
\right)
\right\|^2
\right]
\le
\lambda^2 C_{\mathrm{tilt}}.
\]
\end{proof}

We now combine the excess-risk and tilting-bias controls. To pass from the
pointwise closure in the main text to a full expected-regret bound, we use the
following uniform version of the coverage condition.

\begin{assumption}[Uniform coverage constant]
\label{ass:joint-DW-FM-coverage}
There exists a deterministic constant \(\bar C_{\widehat w}<\infty\) such that,
almost surely over the training procedure and for almost every context
\(\mathbf x\), Assumption~\ref{ass:DW-FM-coverage} holds at
\(\theta=\widehat\theta\) with constants \(A_{\mathbf x}\) and \(B_{\mathbf x}\)
satisfying
\[
A_{\mathbf x}B_{\mathbf x}
\le
\bar C_{\widehat w}.
\]
\end{assumption}

Assumption~\ref{ass:joint-DW-FM-coverage} is the empirical counterpart of
Assumption~\ref{ass:DW-FM-coverage}: it requires the product of the path-overlap
and sensitivity-coverage constants to be uniformly bounded for the learned
field \(\mathbf v_{\widehat\theta}\).

\begin{corollary}
\label{cor:DW-FM-finite-sample}
Under Assumptions~\ref{ass:bounded-DW-FM-loss},
\ref{ass:tilt-cov-plugin}, and~\ref{ass:joint-DW-FM-coverage},
\[
\mathbb E
\!\left[
\mathrm{Reg}_{\mathbf X}(\widehat\theta)
\right]
\le
2
\sqrt{
2\bar C_{\widehat w}
\left[
\mathbb E
\!\left[
\inf_{\theta\in\Theta}
\bar E_{\widehat w}(\mathbf v_\theta)
\right]
+
4\mathfrak R_n(\mathcal F_{\widehat w})
+
\eta
+
\lambda^2 C_{\mathrm{tilt}}
\right]
}.
\]
\end{corollary}

\begin{proof}
Condition on \(\mathcal A_n\). Applying the pointwise regret inequality
\eqref{eq:pointwise-regret} at \(\theta=\widehat\theta\) gives, for almost every
\(\mathbf x\),
\[
\mathrm{Reg}_{\mathbf x}(\widehat\theta)
\le
2
\sqrt{
2C_{\widehat w,\mathbf x}
\left(
E_{\widehat w,\mathbf x}(\mathbf v_{\widehat\theta})
+
B_{\mathrm{tilt},\mathbf x}(\widehat w)
\right)
},
\]
where \(C_{\widehat w,\mathbf x}:=A_{\mathbf x}B_{\mathbf x}\). Using
\(C_{\widehat w,\mathbf x}\le \bar C_{\widehat w}\) and Jensen's inequality over
the fresh test context,
\[
\mathbb E_{\mathbf X}
\!\left[
\mathrm{Reg}_{\mathbf X}(\widehat\theta)
\mid
\mathcal A_n
\right]
\le
2
\sqrt{
2\bar C_{\widehat w}
\left(
\bar E_{\widehat w}(\mathbf v_{\widehat\theta})
+
\bar B_{\mathrm{tilt}}(\widehat w)
\right)
}.
\]
Taking full expectation over the training procedure and applying Jensen's
inequality once more,
\[
\mathbb E
\!\left[
\mathrm{Reg}_{\mathbf X}(\widehat\theta)
\right]
\le
2
\sqrt{
2\bar C_{\widehat w}
\,
\mathbb E
\!\left[
\bar E_{\widehat w}(\mathbf v_{\widehat\theta})
+
\bar B_{\mathrm{tilt}}(\widehat w)
\right]
}.
\]
Theorem~\ref{thm:DW-FM-excess-risk} controls the first term, and
Proposition~\ref{prop:tilting-bias-bound-plugin} controls the second term.
Therefore,
\[
\mathbb E
\!\left[
\bar E_{\widehat w}(\mathbf v_{\widehat\theta})
+
\bar B_{\mathrm{tilt}}(\widehat w)
\right]
\le
\mathbb E
\!\left[
\inf_{\theta\in\Theta}
\bar E_{\widehat w}(\mathbf v_\theta)
\right]
+
4\mathfrak R_n(\mathcal F_{\widehat w})
+
\eta
+
\lambda^2 C_{\mathrm{tilt}}.
\]
Substituting this bound proves the corollary.
\end{proof}
Corollary~\ref{cor:DW-FM-finite-sample} is a full-expectation finite-sample bound. The first term is the expected approximation error; it vanishes in the realizable case, namely when
there exists $\theta^\star\in\Theta$ such that
$\mathbf{v}_{\theta^\star}(\mathbf{Y})=\mathbf{u}_{\widehat w,\mathbf{X}}(\mathbf{Y})$
almost surely. It can be
$O(n^{-1/2})$ under standard sieve or growing-class approximation conditions. The second term is
the full expected Rademacher complexity of the weighted loss class. For bounded-complexity classes,
such as fixed-dimensional parametric models, finite-pseudodimension classes, or norm-controlled
neural-network classes, one typically has
\[
\mathfrak{R}_n(\mathcal{F}_{\widehat w})=O(n^{-1/2}).
\]
The optimization term is $O(n^{-1/2})$ whenever the empirical DW-FM objective is solved to accuracy
$\eta=O(n^{-1/2})$. The tilting term is $O(n^{-1/2})$ if the reweighting strength is chosen as
$\lambda_n=O(n^{-1/4})$. Therefore, under these conditions,
Corollary~\ref{cor:DW-FM-finite-sample} implies
\[
\mathbb{E}
\!\left[
\mathrm{Reg}_{\mathbf{X}}(\widehat\theta)
\right]
=
O(n^{-1/4}).
\]

\section{Experimental Details}
\label{app:exp detail}



\subsection{Synthetic Polynomial Portfolio Task}

\begin{itemize}
    \item \textbf{Context generation:} Context vectors $\mathbf{x} \in \mathbb{R}^{d_x}$ are sampled i.i.d.\ from a standard multivariate Gaussian $\mathcal{N}(\mathbf{0}, I_{d_x})$. The dimensionality $d_x$ corresponds to the number of covariates for the portfolio mapping.

    \item \textbf{Polynomial return map:} For each asset $i \in \{1,\dots,d\}$, the next-period return is generated via a polynomial function of the context:
    \begin{equation}
        f_i(\mathbf{x}) = \sum_{j=1}^{d_x} w_{ij} x_j^{\text{deg}}, 
        \quad \text{deg} \in \{2,4,6,8\}, \quad w_{ij} \sim \text{Uniform}[-1,1]
    \end{equation}
    where $\text{deg}$ is the polynomial degree sweep parameter, and $w_{ij}$ are randomly sampled coefficients fixed per experiment.

    \item \textbf{Stress interpolation:} To control the probability of high-risk events, returns are generated as a mixture between the baseline polynomial map (plus Gaussian residual) and heavy-tail stress:
    \begin{equation}
        s_i = (1-\lambda_\text{stress}) \cdot \big(f_i(\mathbf{x}) + \mathcal{N}(0, \sigma^2)\big) 
              + \lambda_\text{stress} \cdot \epsilon_i,
        \quad \lambda_\text{stress} \in [0.05,0.35]
    \end{equation}
    where $\sigma^2$ is the variance of the baseline Gaussian residual, and $\lambda_\text{stress}$ controls the frequency of stress-tail scenarios. Additive noise $\epsilon_i$ is drawn from a Student-t distribution with degrees of freedom $\nu=3$ and scale $\beta=0.02$ to simulate extreme risk events:
       $ \epsilon_i \sim t_{\nu=3}(\text{scale} = \beta)$.

    \item \textbf{Target outputs:} Next-period portfolio returns $\mathbf{s} \in \mathbb{R}^{d}$, with $d=10$ for standard semi-real portfolio evaluation.

    \item \textbf{Downstream solver:} Long-only portfolio optimization with mean+CVaR objective:
    \[
        \min_\mathbf{z} \mathbb{E}[-\mathbf{s}^\top \mathbf{z}] + \gamma \text{CVaR}_\alpha[-\mathbf{s}^\top \mathbf{z}] + \eta \|\mathbf{z}\|_2^2,
    \]
    subject to simplex constraints $0 \le z_i \le 0.30$, $\sum_i z_i = 1$. Parameters $\alpha=0.95$, $\gamma=2.0$, $\eta=1e-3$.

    \item \textbf{Scenario sampling:} For each context $\mathbf{x}$, $M=512$ scenarios are sampled using the above mixture to generate a distribution of potential returns for downstream optimization.

    \item \textbf{Evaluation metrics:} Regret, hardest-25\% context regret.
\end{itemize}

\subsection{Semi-Real Portfolio-CVaR Task}
\begin{itemize}
    \item \textbf{Data:} Ken French 10 Industry Portfolios returns + Fama-French daily factors. 99-dimensional context features constructed as follows:
    \begin{itemize}
        \item Lagged 1–5 days of each factor
        \item Rolling volatility estimates (20-day)
        \item Industry-specific lagged returns
        \item Standardized using training set mean/std only
    \end{itemize}
    \item \textbf{Target outputs:} Next-day returns $\mathbf{s} \in \mathbb{R}^{10}$.
    \item \textbf{Train/Validation/Test split:} Chronological, 70\% / 10\% / 20\%.
    \item \textbf{Downstream solver:} Long-only portfolio-CVaR optimization with $\gamma=2.0$, $\eta=1e-3$, $\text{CVaR}_{0.90}$.
    \item \textbf{Scenario sampling:} 512 scenarios per context via FM model.
    \item \textbf{Evaluation:} Regret, realized CVaR loss.
\end{itemize}


\subsection{PEMS-BAY Traffic Congestion Task}
\begin{itemize}
    \item \textbf{Data:} PEMS-BAY traffic sensors, 325 sensors, 5-minute frequency, chronological order maintained. Each context $\mathbf{x}_t \in \mathbb{R}^{325}$ represents the current traffic state across all sensors at time $t$.

    \item \textbf{Target outputs:} Next-step congestion load per sensor, $\mathbf{s}_t \in \mathbb{R}^{325}$, defined as
    \[
        s_{t,i} = \max(0, q_{40,i}^\text{train} - \text{speed}_{t,i}),
    \]
    where $q_{40,i}^\text{train}$ is the 40th percentile speed for sensor $i$ in the training set.

     \item \textbf{Train/Validation/Test split:} Chronological split: $70\%/10\%/20\%.$

    \item \textbf{Downstream solver:} CVaR-constrained allocation to minimize aggregate congestion across all sensors:
    \[
        \min_{\mathbf{z}_t} \mathbb{E}\big[ \mathbf{c}_t^\top \mathbf{z}_t \big] + \gamma \, \text{CVaR}_\alpha\big[ \mathbf{c}_t^\top \mathbf{z}_t \big]
    \]
    subject to linear convex constraints $\sum_i z_{t,i} = 1, 0 \le z_{t,i} \le u_i$, and $\mathbf{A} \mathbf{z}_t \le \mathbf{b}$,
    where $\mathbf{z}_t$ represents the allocation of mitigation resources (e.g., signal control adjustments, lane priorities), $\mathbf{c}_t$ is the congestion load vector, $\gamma$ is CVaR weight, $\alpha=0.95$, $u_i$ is per-sensor allocation upper bound, and $\mathbf{A}, \mathbf{b}$ encode additional linear feasibility constraints.

    \item \textbf{Solver implementation:} Uses CVXPY with a linear approximation of CVaR via auxiliary variables. Tolerance set to $10^{-6}$, maximum iterations 1000.
     \item \textbf{Scenario generation:} For DW-FM evaluation, $M=512$ predicted congestion scenarios per context are generated using model sampling.

    \item \textbf{Evaluation metrics:} Regret.

\end{itemize}

\subsection{Training / Model Hyperparameters}
\begin{itemize}
    \item \textbf{Flow Matching model:} 2-layer MLP, width 64, ReLU activation.
    \item \textbf{Optimizer:} Adam, learning rate $1\mathrm{e}{-3}$, weight decay $1\mathrm{e}{-4}$, batch size 256, training steps 400 per smoke experiment, full training steps 200k for main experiments.
    \item \textbf{FM sampling:} base distribution standard Gaussian, linear interpolation, ODE steps = 1.
    \item \textbf{DW-FM:} $\lambda$ grid $\{0, 0.001, 0.002, 0.005, 0.01, 0.02\}$, selected by validation regret on the validation split.
    \item \textbf{Reference decision:} $\hat {\mathbf{z}}_{\mathbf{x}}$ computed via frozen SAA solver; endpoint weights reweighted plug-in via Algorithm 1.
\end{itemize}

\subsection{Compute Resources}
\begin{itemize}
    \item Experiments run on NVIDIA GeForce RTX 4090 GPUs.
    \item Peak GPU memory per card $\approx$ 24GB.
    \item Controlled synthetic degree sweep smoke: 1–2 hours per $\lambda$ per degree.
    \item Semi-real portfolio-CVaR smoke: 2–3 hours per $\lambda$ on 1024 test contexts.
    \item PEMS-BAY traffic smoke: ~0.12–0.14s per context for K=100.
    \item Total aggregate GPU hours across all tasks: $\sim$ 300–400 GPU-hours.
    
\end{itemize}

\section{Limitations}
\label{app:limitations}
Our work has several limitations. First, the theoretical guarantees rely on regularity and coverage assumptions, including smooth transport paths, no-boundary-flux conditions, and path-overlap/sensitivity-coverage between the learned flow path and the interpolation distribution. These assumptions make the regret analysis tractable, but they may be violated in finite-sample training, under heavy-tailed data, or when the learned ODE path moves through regions poorly covered by the training interpolation distribution. Second, DW-FM currently relies on differentiable loss information with respect to the uncertain outcome in order to construct endpoint weights. For non-smooth, discrete, black-box, or simulator-based objectives, additional smoothing, surrogate gradients, or alternative weighting rules may be required.

Moreover, our empirical evaluation focuses on CVaR-based decision problems, with two
portfolio-CVaR benchmarks and one traffic-CVaR benchmark. These tasks support
the effectiveness of decision-weighted training in risk-sensitive CSO, but they
do not establish universal gains across all stochastic optimization problems.

\end{document}